\documentclass{article}
\usepackage[margin=1in]{geometry}

\usepackage{amsmath, amssymb, amsthm, amsfonts}

\newtheorem*{definition}{Definition}
\newtheorem{prop}{Proposition}

\newtheorem{theorem}{Theorem}
\newtheorem{corollary}{Corollary}
\newtheorem{lemma}{Lemma}
\newtheorem*{PL-Inequality}{Polyak-Lojasiewicz (PL) Inequality}
\newtheorem*{theorem*}{Theorem}
\newtheorem*{prop*}{Proposition}
\newtheorem*{lemma*}{Lemma}
\newtheorem*{corollary*}{Corollary}

\usepackage{breqn}
\usepackage{diagbox}
\usepackage{mathtools}
\usepackage{enumerate}
\usepackage{graphicx}
\usepackage{array}
\usepackage{tikz}
\usepackage{hyperref, doi}
\usepackage{epsf}
\usepackage{bbm}
\usepackage[utf8]{inputenc} 
\usepackage[T1]{fontenc}    
\usepackage{float}

\usepackage{caption}
\usepackage{subcaption}
\usepackage{array}
\usepackage{algpseudocode,algorithm,algorithmicx}
\usepackage{makecell}
\usepackage{authblk}
\usepackage{parskip}
\usepackage{wrapfig}
\usepackage{selectp}

\renewcommand{\algorithmicrequire}{\textbf{Input:}}
\renewcommand{\algorithmicensure}{\textbf{Output:}}

\begin{document}

\title{ Linear Convergence of Generalized Mirror Descent with Time-Dependent Mirrors }
\author{Adityanarayanan Radhakrishnan \thanks{Laboratory for Information \& Decision Systems, and 
 Institute for Data, Systems, and Society, 
 Massachusetts Institute of Technology} $~~$ Mikhail Belkin \thanks{
Halıcıoğlu Data Science Institute, University of California, San Diego} $~~$ Caroline Uhler \textsuperscript{*}}
\date{\today}

\maketitle

\begin{abstract}
        The Polyak-Lojasiewicz (PL) inequality is a sufficient condition for establishing linear convergence of gradient descent, even in non-convex settings. While several recent works use a PL-based analysis to establish linear convergence of stochastic gradient descent methods, the question remains as to whether a similar analysis can be conducted for more general optimization methods.  In this work, we present a PL-based analysis for \textit{linear} convergence of \textit{generalized mirror descent} (GMD), a generalization of mirror descent with a possibly time-dependent mirror.  GMD subsumes popular first order optimization methods including gradient descent, mirror descent, and preconditioned gradient descent methods such as Adagrad.  Since the standard PL analysis cannot be extended naturally from GMD to stochastic GMD, we present a Taylor-series based analysis to establish sufficient conditions for linear convergence of stochastic GMD.  As a corollary, our result establishes sufficient conditions and provides learning rates for linear convergence of stochastic mirror descent and Adagrad.  Lastly, for functions that are locally PL\textsuperscript{*}, our analysis implies existence of an interpolating solution and convergence of GMD to this solution.  
\end{abstract}

\section{Introduction}
The Polyak-Lojasiewicz (PL) inequality (Eq.\eqref{PL-Inequality}) serves as a sufficient condition for an elegant proof of linear convergence of stochastic gradient descent methods, even in non-convex settings \cite{PLConditionLinearConvergence, MarkSchmidtSGDPL, BassilySGDLinearConvergence}.  Recently, PL-based analyses have become popular for analyzing convergence of modern machine learning methods.  For example, \cite{PLAndNTKBelkin} used a PL-based analysis as a framework for understanding optimization of over-parameterized neural networks, and \cite{AdagradPL} used a PL-based analysis for establishing linear convergence of Adagrad-Norm.  Since the PL-inequality has served as a powerful tool for establishing linear convergence in several instances, one may wonder how far this tool can be pushed.

In this work, we demonstrate that a PL-based analysis can be used to establish linear convergence for a general class of optimization methods.  Namely, we use a PL-based analysis to establish linear convergence for \emph{generalized mirror descent} (GMD), an extension of mirror descent that introduces (1) a potential-free update rule and (2) a time-dependent mirror.  GMD with invertible $\phi: \mathbb{R}^d \rightarrow \mathbb{R}^{d}$ and learning rate $\eta$  minimizes a real valued function, $f: \mathbb{R}^{d} \to \mathbb{R}$, using the following update rule:
\begin{align}
\label{eq: descent updates introduction}
\phi^{(t)}(w^{(t+1)}) = \phi^{(t)}(w^{(t)}) - \eta \nabla f(w^{(t)}).
\end{align}

We discuss the stochastic version of GMD (SGMD) in Section \ref{sec: Algorithm Description and Preliminaries}.  GMD generalizes both mirror descent and preconditioning methods.  Namely,  if for all $t$, $\phi^{(t)} = \nabla \psi$ for some strictly convex function $\psi$, then the update rule in equation \eqref{eq: descent updates introduction} reduces to:
\begin{align*}
    \nabla \psi(w^{(t+1)}) = \nabla \psi(w^{(t)}) - \eta \nabla f(w^{(t)})
\end{align*}
and hence GMD corresponds to mirror descent with potential $\psi$.  If $\phi^{(t)} = G^{(t)}$ for some invertible matrix $G^{(t)} \in \mathbb{R}^{d \times d}$, then the update rule in equation \eqref{eq: descent updates introduction} reduces to
\begin{align*}
    w^{(t+1)} = w^{(t)} - \eta {G^{(t)}}^{-1} \nabla f(w^{(t)})
\end{align*}
and hence represents applying a pre-conditioner to gradient updates.  Thus, by using a PL-based analysis to establish linear convergence of GMD and SGMD, we establish linear convergence of both mirror descent and pre-conditioner methods.  The following is a summary of our results:

\begin{enumerate}
    \item We provide sufficient conditions under which the standard PL-based analysis for gradient descent can be extended to establish linear convergence of GMD (Theorem~\ref{thm: Theorem 1 - Linear Convergence of GGD}).
    \item While the standard PL-based analysis cannot be extended naturally to the stochastic setting, we use a Taylor-series based analysis to provide sufficient conditions under which SGMD converges linearly (Theorems \ref{thm: Theorem 2 - Taylor Series Analysis} and \ref{thm: Stochastic GGD linear convergence}).  
    \item As corollaries to Theorems \ref{thm: Theorem 1 - Linear Convergence of GGD} and \ref{thm: Stochastic GGD linear convergence}, in Section \ref{sec: Corollaries}, we provide sufficient conditions for linear convergence of stochastic mirror descent as well as stochastic preconditioner methods such as Adagrad \cite{Adagrad}.
    \item We prove the existence of an interpolating solution and linear convergence of GMD to this solution for non-negative loss functions that locally satisfy the PL\textsuperscript{*} inequality~\cite{PLAndNTKBelkin}.  This result generalizes the local convergence result (Theorem 4.2) from \cite{PLAndNTKBelkin} for gradient descent, see also~\cite{PLGDShortestPath}.  Lastly, for the case of mirror descent, our result provides a formula for the radius of the ball around the initialization in Bregman divergence that contains an interpolating solution.  A bound on this radius was a key assumption used in the approximate implicit regularization results of \cite{MirrorDescentNonlinear}.  
\end{enumerate}

\section{Related Work}

The Polyak-Lojasiewicz (PL) inequality \cite{PLInequalityLojasiewicz, PLInequality} serves as a simple condition for linear convergence in non-convex optimization problems.  Work by \cite{PLConditionLinearConvergence} demonstrated linear convergence of a number of descent methods (including gradient descent) under the PL inequality.  Similarly, \cite{MarkSchmidtSGDPL} proved linear convergence of stochastic gradient descent (SGD) under the PL inequality and the strong growth condition (SGC), and \cite{BassilySGDLinearConvergence} established the same rate for SGD under just the PL inequality.  \cite{MahdiPLOptimizationLandscape} also used the PL inequality to establish a local linear convergence result for gradient descent on one hidden layer over-parameterized neural networks.  A PL-based analysis of pseudo mirror descent for estimating positive functions in a Hilbert space was presented in Theorem 6 of  \cite{PLPseudoMirrorDescent}.  The assumptions used to establish linear convergence by \cite{PLPseudoMirrorDescent}, however, are much stronger than ours (i.e. assuming a non-trivial lower bound when using smoothness of the loss) and also apply to pseudo-gradients instead of exact gradients.  Recent works \cite{PLGDShortestPath, PLAndNTKBelkin} also use a local PL-based analysis to analyze the convergence of gradient descent for over-parameterized nonlinear systems such as neural networks.

Recent work \cite{MirrorDescentNonlinear} established convergence of stochastic mirror descent (SMD) for nonlinear optimization problems.  It characterized approximate implicit bias of mirror descent by demonstrating that SMD converges to a global minimum that is within epsilon of the closest interpolating solution in Bregman divergence.  The analysis by \cite{MirrorDescentNonlinear} relies on the fundamental identity of SMD and does not provide explicit learning rates or establish a rate of convergence for SMD in the nonlinear setting.  Importantly, the analysis by \cite{MirrorDescentNonlinear} requires the assumption that there exists a ball in Bregman divergence around the initialization that contains an interpolating solution.  We provide an explicit bound on the radius of this ball through our results in Section \ref{sec: Implicit Regularization}.   The work by \cite{MirrorDescentMinimax}  provided explicit learning rates for the convergence of SMD in the linear setting under strongly convex potential, again without a rate of convergence.  While these works established convergence of SMD, prior work by \cite{ImplicitBiasOptimizationGeometry} analyzed the implicit bias of SMD without proving convergence.

 A potential-based version of generalized mirror descent with time-varying regularizers was presented for online problems by \cite{GeneralizedMirrorDescent}.  This work is primarily concerned with establishing regret bounds for the online learning setting, which differs from our setting of minimizing a loss function given a set of known data points.  A potential-free formulation of GMD for the flow was presented by \cite{MirrorlessDescent}.

Recently, \cite{AdagradPL} established linear convergence for a norm version of Adagrad, i.e. Adagrad-Norm, using the PL inequality and \cite{AdagradNeuralNetwork} established linear convergence for Adagrad-Norm in the particular setting of over-parameterized neural networks with one hidden layer.  An alternate analysis for Adagrad-Norm for smooth, non-convex functions was presented by \cite{AdagradNormOriginalAnalysis}, resulting in a sub-linear convergence rate. 

Instead of focusing on a specific method, the goal of this work is to establish sufficient conditions for linear convergence by applying the PL inequality to a more general setting (SGMD).  We arrive at linear convergence for specific methods such as mirror descent and preconditioned gradient descent methods as corollaries.  In addition, our local convergence result (Theorem \ref{thm: Local convergence of GMD}) establishes the existence of an interpolating solution in a ball (in Bregman divergence) around the initialization, which was used as a key assumption in the results of \cite{MirrorDescentNonlinear}.

\section{Algorithm Description and Preliminaries}
\label{sec: Algorithm Description and Preliminaries}
We begin with a formal description of SGMD.  Let $f_i: \mathbb{R}^d \rightarrow \mathbb{R}$ denote real-valued, differentiable loss functions and let $f(x) = \frac{1}{n}\sum_{i=1}^{n}  f_i(x)$. In addition, let $\phi^{(t)}: \mathbb{R}^d \rightarrow \mathbb{R}^d$ be an invertible function for all non-negative integers $t$.  We solve the optimization problem
\begin{align*}
    \arg \min_{x \in \mathbb{R}^d} f(x) 
\end{align*}
using \textbf{stochastic generalized mirror descent} with learning rate $\eta$\footnote{The framework also allows for adaptive learning rates by using $\eta^{(t)}$ to denote a time-dependent step size.}:
\begin{align}
\label{eq: descent updates}
\phi^{(t)}(w^{(t+1)}) = \phi^{(t)}(w^{(t)}) - \eta \nabla f_{i_t}(w^{(t)}),
\end{align}
where $i_t \in [n]$ is chosen uniformly at random.  As described in the introduction, the above algorithm generalizes both gradient descent (where $\phi(x) = x$) and mirror descent (where $\phi^{(t)}(x) = \nabla \psi(x)$ for some strictly convex potential function $\psi$).  In the case where $\phi^{(t)}(x) = G^{(t)} x$ for an invertible matrix $G^{(t)} \in \mathbb{R}^{d \times d}$, the update rule in equation \eqref{eq: descent updates} reduces to:
\begin{align*}
    w^{(t+1)} = w^{(t)} - \eta {G^{(t)}}^{-1} \nabla f_{i_t}(w^{(t)}).
\end{align*}
Hence, when $\phi^{(t)}$ is an invertible linear transformation, Equation (\ref{eq: descent updates}) is equivalent to pre-conditioned gradient descent.  We now present the Polyak-Lojasiewicz inequality \cite{PLInequality} and lemmas from optimization theory that will be used in our proofs\footnote{We assume all norms are the $2$-norm unless stated otherwise.}.  

\begin{PL-Inequality}
A function $f: \mathbb{R}^d \rightarrow \mathbb{R}$ is $\mu$-PL if for some $\mu > 0$:
\begin{equation}
    \label{PL-Inequality}
    \frac{1}{2}\|\nabla f(x) \|^2 \geq \mu (f(x) - f(x^*)) ~~ \forall x \in \mathbb{R}^d,  
\end{equation}
 where $x^* \in \mathbb{R}^d$ is a global minimizer for $f$. 
\end{PL-Inequality}

A useful variation of the PL inequality is the PL\textsuperscript{*} inequality using the terminology from \cite{PLAndNTKBelkin} 
which does not require knowledge of $f(x^*)$. 

\begin{definition}
\label{def: PL*}
A function $f: \mathbb{R}^d \rightarrow \mathbb{R}$ is $\mu$-PL\textsuperscript{*} if for some $\mu > 0$:
\begin{equation}
    \frac{1}{2}\|\nabla f(x) \|^2 \geq \mu f(x)  ~~ \forall x \in \mathbb{R}^d.  
\end{equation}
\end{definition}

A function that is $\mu$-PL\textsuperscript{*} is also $\mu$-PL when $f$ is non-negative.     In this work, we will typically assume that $f$ is $L$-smooth (with $L$-Lipschitz continuous derivative).

\begin{definition}
\label{def: L-smoothness}
A function $f:\mathbb{R}^{d} \rightarrow \mathbb{R}$ is $L$-smooth for $L > 0$ if for all $x, y \in \mathbb{R}^{d}$: 
\begin{align*}
    \|\nabla f(x) - \nabla f(y) \| \leq L \|x - y\|.
\end{align*}
\end{definition}

If $\phi^{(t)}(x) = x$ for any $t$ and $x \in \mathbb{R}^d$  then SGMD reduces to SGD. If $f$ is $L$-smooth and satisfies the PL Inequality, then SGD converges linearly to a global minimum \cite{PLInequality, BassilySGDLinearConvergence, PLConditionLinearConvergence, MarkSchmidtSGDPL}.  Moreover, the following lemma shows that the PL\textsuperscript{*} condition implies the existence of a global minimum $x^*$ for non-negative, $L$-smooth $f$ (the proof follows from immediately from the analysis in \cite{PLInequality}). 
\begin{lemma}
\label{lemma: Linear convergence under PL*}
If $f: \mathbb{R}^d \rightarrow \mathbb{R}$ is $\mu$-PL\textsuperscript{*}, $L$-smooth and $f(x) \geq 0$ for all $x \in \mathbb{R}^{d}$, then gradient descent with learning rate $\eta < \frac{2}{L}$ converges linearly to $x^*$ satisfying $f(x^*) = 0$. 
\end{lemma}

Hence, in cases where the loss function is nonnegative (for example the squared loss), we can remove the usual assumption about the existence of a global minimum, $x^*$, and instead assume that $f$ satisfies the PL\textsuperscript{*} inequality.  We now reference standard properties of $L$-smooth functions \cite{NesterovConvexity}, which will be used in our proofs.
\begin{lemma}
\label{lemma: Properties of $L$-smooth functions}
If $f:\mathbb{R}^{d} \rightarrow \mathbb{R}$ is $L$-smooth, then for all $x, y \in \mathbb{R}^d$:
\begin{align*}
    & (a) ~~ f(y) \leq f(x) + \langle \nabla f(x), y - x \rangle + \frac{L}{2} \| y - x\|^2, \\
    & (b) ~~ \| \nabla f(x)\|^2 \leq 2L (f(x) - f(x^*)).
\end{align*}
\end{lemma}

The following lemma relates $\mu$ and $L$ and follows from the the PL-inequality and Lemma \ref{lemma: Properties of $L$-smooth functions}b \cite{NesterovConvexity}.

\begin{lemma}
\label{lemma: mu less than L}
If $f: \mathbb{R}^d \rightarrow \mathbb{R}$ is $\mu$-PL and $L$-smooth, then $\mu \leq L$.  
\end{lemma}

Using Lemma \ref{lemma: Properties of $L$-smooth functions}b in place of the strong growth condition (i.e. $\mathbb{E}_{i}[\|\nabla f_i(x)\|^2] \leq \rho \|\nabla f(x)\|^2$) \cite{MarkSchmidtSGDPL} yields slightly different learning rates when establishing convergence of stochastic descent methods (as is apparent from the different learning rates between \cite{BassilySGDLinearConvergence} and \cite{MarkSchmidtSGDPL}).  The following simple lemma will be used in the proof of Theorem \ref{thm: Stochastic GGD linear convergence}.  

\begin{lemma}
\label{lemma: sum is smooth}
If $f(x) = \frac{1}{n}\sum_{i=1}^n f_i(x)$ where $f_i: \mathbb{R}^{d} \rightarrow \mathbb{R}$ are $L_i$-smooth , then $f$ is $\sup_{i} L_i$-smooth.  
\end{lemma}
Note that there could exist some other constant $L' < \sup_{i}L_i$ for which $f$ is $L'$-smooth, but this upper bound suffices for our proof of Theorem \ref{thm: Stochastic GGD linear convergence}.  Lastly, we define and reference standard properties of strongly convex functions \cite{NesterovConvexity}, which will be useful in connecting our local convergence result with Assumption 1 by \cite{MirrorDescentNonlinear}.  

\begin{definition}
\label{def: Strong Convexity}
    For $\alpha > 0$, a differentiable function, $\psi: \mathbb{R}^{d} \to \mathbb{R}$, is $\alpha$-strongly convex if for all $x, y$,
    \begin{align*}
        \psi(y) \geq \psi(x) + \langle \nabla \psi(x), y - x \rangle + \frac{\alpha}{2} \|y - x\|^2.
    \end{align*}
\end{definition}

\begin{lemma}
\label{lemma: Strong Convexity Upper Bound}
    If $\psi: \mathbb{R}^{d} \to \mathbb{R}$ is $\alpha$-strongly convex, then for all $x, y$:
    \begin{align*}
        \psi(y) \leq \psi(x) + \langle \nabla \psi(x), y - x \rangle + \frac{1}{2\alpha} \|\nabla \psi(y) - \nabla \psi(x)\|^2.
    \end{align*}
\end{lemma}

With these preliminaries in hand, we now present our proofs for linear convergence of SGMD using the PL Inequality.

\section{Sufficient Conditions for Linear Convergence of SGMD}
In this section, we provide sufficient conditions to establish (expected) linear convergence for (stochastic) GMD.  We first provide simple conditions under which GMD converges linearly by extending the proof strategy of \cite{PLConditionLinearConvergence}. Since this proof does not naturally extend to the stochastic setting, we then  present alternate conditions for linear convergence of GMD, which can be naturally extended to the stochastic setting.  

\subsection{Simple Conditions for Linear Convergence of GMD}
We begin with a simple set of conditions under which (non-stochastic) GMD converges linearly. The main benefit of this analysis is that it is a straightforward extension of the classical proof of linear convergence for gradient descent under the PL Inequality  \cite{PLInequality}.  


\begin{theorem}
\label{thm: Theorem 1 - Linear Convergence of GGD}
Suppose $f: \mathbb{R}^{d} \rightarrow \mathbb{R}$ is $L$-smooth and $\mu$-PL and $\phi^{(t)}: \mathbb{R}^d \rightarrow \mathbb{R}^d$ is an invertible, $\alpha_u^{(t)}$-Lipschitz function where $\lim\limits_{t \to \infty} \alpha_u^{(t)} < \infty$.  If for all $x, y \in \mathbb{R}^d$ and for all timesteps $t$ there exist $\alpha_l^{(t)} > 0$ such that
\begin{align*}
    \langle \phi^{(t)}(x) - \phi^{(t)}(y), x - y \rangle \geq  \alpha_l^{(t)} \| x -y \|^2, 
\end{align*}
and $\lim\limits_{t \to \infty} \alpha_l^{(t)} > 0$, then generalized mirror descent converges linearly to a global minimum for any $\eta^{(t)} < \frac{2\alpha_l^{(t)}}{L}$. 
\end{theorem}
\begin{proof}
Since $f$ is $L$-smooth, Lemma \ref{lemma: Properties of $L$-smooth functions}a implies:
\begin{align}
\label{eq: L-smooth}
    f(w^{(t+1)}) -  f(w^{(t)}) \leq \langle \nabla f(w^{(t)}), w^{(t+1)} - w^{(t)} \rangle + \frac{L}{2} \| w^{(t+1)} - w^{(t)} \|^2.
\end{align}
Now by the condition on $\phi^{(t)}$ in Theorem \ref{thm: Theorem 1 - Linear Convergence of GGD}, we bound the first term on the right as follows:
\begin{align*}
     & \langle \phi^{(t)}(w^{(t+1)}) - \phi^{(t)}(w^{(t)}), w^{(t+1)} - w^{(t)} \rangle \geq  \alpha_l^{(t)} \| w^{(t+1)} - w^{(t)} \|^2 \\
     \implies & \langle -\eta \nabla f(w^{(t)}), w^{(t+1)} - w^{(t)} \rangle \geq  \alpha_l^{(t)} \| w^{(t+1)} - w^{(t)} \|^2 ~~ \text{using Equation \eqref{eq: descent updates}}\\
     \implies & \langle \nabla f(w^{(t)}), w^{(t+1)} - w^{(t)} \rangle \leq  -\frac{\alpha_l^{(t)}}{\eta} \| w^{(t+1)} - w^{(t)} \|^2.
\end{align*}
Substituting this bound back into the inequality in \eqref{eq: L-smooth}, we obtain:
\begin{align*}
    f(w^{(t+1)}) -  f(w^{(t)}) \leq \left( -\frac{\alpha_l^{(t)}}{\eta} + \frac{L}{2} \right) \| w^{(t+1)} - w^{(t)} \|^2.
\end{align*}
Since the learning rate is selected so that the coefficient of $\| w^{(t+1)} - w^{(t)} \|^2 $ is negative, we obtain:
\begin{align*}
    f(w^{(t+1)}) -  f(w^{(t)})  &\leq \left( -\frac{\alpha_l^{(t)}}{\eta} + \frac{L}{2} \right) \| w^{(t+1)} - w^{(t)} \|^2 \\
    &\leq \left( -\frac{\alpha_l^{(t)}}{\eta} + \frac{L}{2} \right) \frac{1}{{\alpha_u^{(t)}}^2} \|\phi^{(t)}(w^{(t+1)}) - \phi^{(t)}(w^{(t)}) \|^2 \\ 
    &= \left( -\frac{\alpha_l^{(t)}}{\eta} + \frac{L}{2} \right) \frac{1}{{\alpha_u^{(t)}}^2} \|-\eta \nabla f(w^{(t)}) \|^2 ~~ \text{using Eq.\eqref{eq: descent updates introduction}}\\
    &\leq \left( -\frac{\alpha_l^{(t)}}{\eta} + \frac{L}{2} \right) 2 \mu \frac{\eta^2}{{\alpha_u^{(t)}}^2} (f(w^{(t)}) - f(w^*)) ~~ \text{as $f$ is $\mu$-PL}\\
\implies f(w^{(t+1)}) -  f(w^*)  &\leq \left(  1 - 2 \mu \frac{\eta\alpha_l^{(t)}}{{\alpha_u^{(t)}}^2} + \mu \frac{L\eta^2}{{\alpha_u^{(t)}}^2} \right)(f(w^{(t)}) - f(w^*)),
\end{align*}
where the second inequality follows since $\phi^{(t)}$ is $\alpha_u^{(t)}$-Lipschitz.  For linear convergence, we need. 
\begin{align}
\label{eq_1}
    0 <  1 - 2 \mu \frac{\eta\alpha_l^{(t)}}{{\alpha_u^{(t)}}^2} + \mu \frac{L\eta^2}{{\alpha_u^{(t)}}^2} < 1.
\end{align}
From Lemma \ref{lemma: mu less than L}, $\mu < \frac{{\alpha_u^{(t)}}^2 L}{\alpha_l^{(t)}}$ always holds and implies that the left  inequality in (\ref{eq_1}) is satisfied for all $\eta^{(t)}$.  The right inequality holds by our assumption that $\eta^{(t)}< \frac{2\alpha_l^{(t)}}{L}$, which completes the proof.
\end{proof}

\textbf{Remark.}  Theorem~\ref{thm: Theorem 1 - Linear Convergence of GGD} yields a fixed learning rate provided that $\alpha_l^{(t)}$ is uniformly bounded away from $0$.  In addition, note that Theorem \ref{thm: Theorem 1 - Linear Convergence of GGD} applies also under weaker assumptions, namely when $\phi^{(t)}$ is locally Lipschitz.  Finally, the provided learning rate can be computed exactly for settings such as linear regression, since it only requires knowledge of $L$ and $\alpha_l^{(t)}$ (see Section \ref{sec: Experiments}).  When $\eta = \frac{\alpha_l^{(t)}}{L}$ and given $w^*$ a minimizer of $f$, the proof of Theorem \ref{thm: Theorem 1 - Linear Convergence of GGD} implies that: 
\begin{align*}
    f(w^{(t+1)}) -  f(w^*)  &\leq \left(  1 -  \frac{\mu{\alpha_l^{(t)}}^2}{{L\alpha_u^{(t)}}^2} \right)(f(w^{(t)}) - f(w^*)).
\end{align*}
Letting $\kappa^{(t)} = \frac{L{\alpha_u^{(t)}}^2}{\mu{\alpha_l^{(t)}}^2}$ thus generalizes the condition number introduced in Definition 4.1 of \cite{PLAndNTKBelkin} for gradient descent.  Provided that $\inf \left(\frac{1}{\kappa^{(t)}}\right) = \frac{1}{\kappa} > 0$, then Theorem \ref{thm: Theorem 1 - Linear Convergence of GGD} guarantees linear convergence to a global minimum with the rate given by: 
\begin{align*}
    f(w^{(t+1)}) -  f(w^*)  &\leq \left(  1 -  \frac{1}{\kappa} \right)^{t+1}(f(w^{(0)}) - f(w^*)).
\end{align*}

Lastly, as we will discuss in Section \ref{sec: Corollaries}, Theorem \ref{thm: Theorem 1 - Linear Convergence of GGD} generalizes the PL analysis by \cite{PLConditionLinearConvergence}. In particular, for gradient descent, $\alpha_u^{(t)} = \alpha_l^{(t)} = 1$, and we recover the learning rate of $\eta < \frac{2}{L}$.  For the case of mirror descent with potential $\psi$, the lower bound assumption in Theorem \ref{thm: Theorem 1 - Linear Convergence of GGD} corresponds to strong convexity of the potential $\psi$.  This  assumption is also used in the convergence results by \cite{PLPseudoMirrorDescent}.

\subsection{Taylor Series Analysis for Linear Convergence in GMD}

Although the proof of Theorem \ref{thm: Theorem 1 - Linear Convergence of GGD} is succinct, it is nontrivial to extend to the stochastic setting. The main difficulty is relating $w^{(t+1)} - w^{(t)}$ to the gradient at step $t$.  Ideally, a simple argument would come from establishing linear convergence for the terms $z^{(t)} = \phi(w^{(t)})$.  However, this requires several strong and seemingly un-intuitive assumptions to even establish the smoothness of the dual loss $g(z) = f(\phi^{-1}(z))$.  Indeed, for a simple class of functions, e.g. $\phi^{-1}(x) = \alpha(2x + \sin x)$, $g(z)$ need not even be $L$-smooth (See Supplementary \ref{appendix I: alternate conditions for linear convergence in SGMD}).  In order to develop a convergence result for the stochastic setting, we turn to an alternate set of conditions for linear convergence by using the Taylor expansion of $\phi^{-1}$.  We use $\mathbf{J}_{\phi}$ to denote the Jacobian of $\phi$.  For ease of notation, we consider non-time-dependent $\alpha_l, \alpha_u$, but our results are easily extended to the setting when these quantities are time-dependent.  

\begin{theorem}
\label{thm: Theorem 2 - Taylor Series Analysis}
Suppose $f: \mathbb{R}^{d} \rightarrow \mathbb{R}$ is $L$-smooth and $\mu$-PL and $\phi: \mathbb{R}^{d} \rightarrow \mathbb{R}^d$ is an  analytic function with analytic inverse, $\phi^{-1}$.  If there exist $\alpha_l, \alpha_u > 0$ such that
\begin{align*}
    &(a) ~~ \alpha_l \mathbf{I}  \preceq \mathbf{J}_{\phi} \preceq \alpha_u \mathbf{I},\\
    &(b) ~~ |\partial_{i_1, \ldots i_k} \phi_j^{-1} (x)| \leq \frac{k!}{2\alpha_u d},
\end{align*}
$\forall x \in \mathbb{R}^d, \{i_r\}_{r=1}^k \in [d], j \in [d], k \geq 2$, then generalized mirror descent converges linearly for any $\eta^{(t)} < \min\left( \frac{4\alpha_l^2}{5L \alpha_u}, \frac{1}{2\sqrt{d} \|\nabla f(w^{(t)}) \|} \right)$.  In particular, if $\frac{2\alpha_l^2}{5L\alpha_u} < \frac{1}{2\sqrt{2 L d f(w^{(0)})}}$ for $f$ that is $\mu$-PL\textsuperscript{*} and non-negative, then 
\begin{align*}
    f(w^{(t+1)}) \leq \left(1 - \frac{\mu^2 \alpha_l^2}{5 L^2 \alpha_u^2} \right)^{t+1} f(w^{(0)}) 
\end{align*}
 for $\eta = \frac{2\alpha_l^2}{5L\alpha_u}$.
\end{theorem}

The full proof is provided in Supplementary \ref{appendix B: Proof of Theorem 2}, but we provide a sketch of the key ideas below.  
\begin{proof}[Proof Sketch]
We proceed as in the proof of Theorem \ref{thm: Theorem 1 - Linear Convergence of GGD} using the $L$-smoothness of $f$: 
\begin{align*}
    f(w^{(t+1)}) -  f(w^{(t)}) &\leq \langle \nabla f(w^{(t)}), w^{(t+1)} - w^{(t)} \rangle + \frac{L}{2} \| w^{(t+1)} - w^{(t)} \|^2.    
\end{align*}
In order to bound the term $\langle \nabla f(w^{(t)}), w^{(t+1)} - w^{(t)} \rangle$, we note that
\begin{align*}
    w^{(t+1)} - w^{(t)} = \phi^{-1}(\phi(w^{(t)}) -\eta \nabla f(w^{(t)})) - w^{(t)}
\end{align*} 
and use the Taylor series expansion for $\phi^{-1}$ around $\phi(w^{(t)})$ to simplify this expression.  

In the expansion, the constant term is $\phi^{-1}(\phi(w^{(t)})) = w^{(t)}$, which is cancelled by $-w^{(t)}$.  Keeping the first order term and bounding the higher order terms, we conclude
\begin{align*}
    \langle \nabla f(w^{(t)}), w^{(t+1)} - w^{(t)} \rangle  \leq \left( -\frac{\eta}{2\alpha_u}\right) \|\nabla f(w^{(t)})\|^2.
\end{align*}
Proceeding analogously for $\|w^{(t+1)} - w^{(t)}\|^2$, we conclude
\begin{align*}
    \|w^{(t+1)} - w^{(t)}\|^2 \leq \left(\frac{\eta^2}{\alpha_l^2} + \frac{\eta^2}{4\alpha_u^2} \right) \|\nabla f(w^{(t)}) \|^2.
\end{align*}
Lastly, we proceed using the PL inequality analysis from Theorem \ref{thm: Theorem 1 - Linear Convergence of GGD} to complete the proof.  
\end{proof}

\textbf{Remarks.} We note that condition (b) of Theorem \ref{thm: Theorem 2 - Taylor Series Analysis} is satisfied by a large class of functions (e.g. functions with polynomial inverse).  The adaptive component of the learning rate is only used to ensure that the sum of the higher order terms for the Taylor expansion converges.  In particular, if $\phi^{(t)}$ is linear, then our learning rate no longer needs to be adaptive.  Note that alternatively, we can establish linear convergence for a fixed learning rate given that the gradients monotonically decrease or if $f$ is non-negative and $\mu$-PL\textsuperscript{*}.  We analyze the case of monotonically decreasing gradients in Supplementary \ref{appendix C: Conditions for Monotonically Decreasing Gradients} and provide an explicit condition on $\mu$ and $L$ under which this holds. The conditions of Theorem~\ref{thm: Theorem 2 - Taylor Series Analysis} hold for a nontrivial class of functions $\phi$.  As an example, consider the class of functions $\phi^{-1}(x) = \alpha (2x + \sin x)$ from before that acts element-wise on $x \in \mathbb{R}^{d}$.  Even though this class of functions does not satisfy the condition that the dual loss function $g(z) = f(\phi^{-1}(z))$ is $L$-smooth for all $d$, it satisfies the conditions of Theorem \ref{thm: Theorem 2 - Taylor Series Analysis} for $d =1$. In particular, for such $\phi$, by the inverse function theorem we have for condition (a) that $ \frac{1}{3\alpha}\mathbf{I} \preceq \mathbf{J}_{\phi} \preceq \frac{1}{\alpha} \mathbf{I}$  and for condition (b) that $|\partial_{i_1, \ldots, i_k} \phi^{-1}_j(x))| \leq \frac{k!\alpha}{2}$.



\subsection{Taylor Series Analysis for Linear Convergence in Stochastic GMD}

The main benefit of the above Taylor series analysis is that it naturally extends to the stochastic setting as demonstrated in the following result (with proof presented in Supplementary \ref{appendix D: Proof of Theorem 3}).
\begin{theorem}
\label{thm: Stochastic GGD linear convergence}
Suppose $f(x) = \frac{1}{n}\sum_{i=1}^n f_i(x)$ with $f_i: \mathbb{R}^{d} \rightarrow \mathbb{R}$ non-negative, $L_i$-smooth functions with $L = \sup_{i \in [n]} L_i$, $f$ is $\mu$-PL\textsuperscript{*}, and $f(x) = 0$ implies $f_i(x) = 0$ for all $i$.  Let $\phi: \mathbb{R}^{d} \rightarrow \mathbb{R}^d$ be an  analytic function with analytic inverse, $\phi^{-1}$.  SGMD minimizes $f$ using the updates:
\begin{align*}
    \phi(w^{(t+1)}) = \phi(w^{(t)}) - \eta^{(t)} \nabla f_{i_t}(w^{(t)}),
\end{align*}
where $i_t \in [n]$ is chosen uniformly at random and $\eta^{(t)}$ is an adaptive step size. If there exist $\alpha_l, \alpha_u > 0$ such that:
\begin{align*}
    &(a) ~~ \alpha_l \mathbf{I}  \preceq \mathbf{J}_{\phi} \preceq \alpha_u \mathbf{I},\\
    &(b) ~~ |\partial_{i_1, \ldots i_k} \phi_j^{-1} (x)| \leq \frac{k! \mu}{2\alpha_u d L}
\end{align*}
$\forall x \in \mathbb{R}^d, i_1, \ldots i_k \in [d], j \in [d], k \geq 2$, then SGMD converges linearly if $\eta^{(t)} < \min\left( \frac{4\mu\alpha_l^2}{5L^2 \alpha_u}, \frac{1}{2\sqrt{d} \max_i \|\nabla f_{i}(w^{(t)}) \|} \right)$. In particular, if $\frac{2\mu\alpha_l^2}{5L^2 \alpha_u} < \frac{1}{2 \sqrt{2dnL f(w^{(0)})}}$, then
\begin{align*}
    \mathbb{E}_{i_t, \ldots, i_1}[f(w^{(t+1)})] \leq \left(1 - \frac{\alpha_l^2 \mu^2}{5 \alpha_u^2 L^2} \right)^{t+1} f(w^{(0)})
\end{align*}
for $\eta = \frac{2\mu \alpha_l^2}{5L^2 \alpha_u^2}$.
\end{theorem}

The proof proceeds similarly to that of Theorem \ref{thm: Theorem 2 - Taylor Series Analysis} with a few important distinctions.  In order to ensure that the learning rate $\eta^{(t)}$ does not depend on the choice of example $i_t$, we take $\eta^{(t)} \leq \frac{1}{2\sqrt{d} \max_{i} \|\nabla f_i(w^{(t)})\|}$ (which can be made constant by applying Lemma \ref{lemma: Properties of $L$-smooth functions}b).  We also need to apply Lemma \ref{lemma: Properties of $L$-smooth functions}b (instead of the strong growth condition) to the terms involving $\|\nabla f_{i_t}(w^{(t)})\|$ such that we can take expectations to recover terms involving $f(w^{(t)})$.  

\textbf{Remark.} Note that there is a slight difference between the learning rate in  Theorem \ref{thm: Theorem 2 - Taylor Series Analysis} and Theorem~\ref{thm: Stochastic GGD linear convergence} due to a multiplicative factor of $\mu$.  Consistent with the difference in learning rates between \cite{BassilySGDLinearConvergence} and \cite{MarkSchmidtSGDPL}, we can make the learning rate between the two theorems match if we assume the strong growth condition (i.e. $\mathbb{E}_{i}[\|\nabla f_i(x)\|^2] \leq \rho \|\nabla f(x)\|^2$) with $\rho = \mu$ instead of using Lemma \ref{lemma: Properties of $L$-smooth functions}b.  Moreover, since $\max_i \|\nabla f_i(w^{(t)})\| \leq \sqrt{2nL f(w^{(0)})}$, we establish linear convergence for a fixed step size $\eta < \min\left( \frac{4 \mu \alpha_l^2}{5 L^2 \alpha_u}, \frac{1}{2 \sqrt{2dnL f(w^{(0)})}} \right)$ as well.  

We now briefly discuss why the PL\textsuperscript{*} instead of the PL condition is required for the analysis in the stochastic case.  Without the PL\textsuperscript{*} condition, by Lemma \ref{lemma: Properties of $L$-smooth functions}, we would still have that: 
\begin{align*}
    \|\nabla f_{i_t}(w^{(t)})\|^2 \leq 2 L (f_{i_t}(w^{(t)}) - f_{i_t}(w_{i_t}^*))
\end{align*}
However, now each function $f_{i_t}$ has a different minimizer $w_{i_t}^*$, which need not correspond to the minimizer, $w^*$, of $f$.  As an example, we can consider the loss 
$$f(x) = (x - 1)^2 + (x - 2)^2$$  
In this case, the minimizer of $f$ is $x = 1.5$ while the minimizer of $f_1$ is $x = 1$ and that of $f_2$ is $x=2$, and hence, it is not necessarily true that $\|\nabla f_i(x)\|^2 \leq 2 L (f_i(x) - f(1.5))$.

\section{Corollaries of Linear Convergence in SGMD}
\label{sec: Corollaries}
We now present how the linear convergence results established by Theorems \ref{thm: Theorem 1 - Linear Convergence of GGD}, \ref{thm: Theorem 2 - Taylor Series Analysis}, and \ref{thm: Stochastic GGD linear convergence} apply to commonly used optimization algorithms including mirror descent and Adagrad.  In this section, we primarily extend the analysis from Theorem \ref{thm: Theorem 1 - Linear Convergence of GGD} for the non-stochastic case.  However, our results can be extended analogously to give expected linear convergence in the stochastic case by using the extension provided in Theorem \ref{thm: Stochastic GGD linear convergence}.  

\textbf{Gradient Descent.} For the case of gradient descent, $\phi(x) = x$ and so $\alpha_l = \alpha_u = 1$.  Hence, we see that gradient descent converges linearly under the conditions of Theorem \ref{thm: Theorem 1 - Linear Convergence of GGD} with $\eta < \frac{2}{L}$, which is consistent with the analysis by \cite{PLConditionLinearConvergence}.    

\textbf{Mirror Descent.}  Let $\psi: \mathbb{R}^d \rightarrow \mathbb{R}$ be a strictly convex potential.  Thus, $\phi(x) = \nabla \psi (x)$ is an invertible function.  If $\psi$ is $\alpha_l$-strongly convex and (locally) $\alpha_u$-Lipschitz and $f$ is $L$-smooth and $\mu$-PL, then the conditions of Theorem~\ref{thm: Theorem 1 - Linear Convergence of GGD} are satisfied. Moreover, since the $\alpha_u$-Lipschitz condition holds locally for most potentials considered in practice, our result essentially implies linear convergence for mirror descent with $\alpha_l$-strongly convex potential $\psi$.  


\textbf{Adagrad.}  Let $\phi^{(t)} = {\mathcal{G}^{(t)}}^{\frac{1}{2}}$ where $\mathcal{G}^{(t)}$ is a diagonal matrix such that
\begin{align*}
    \mathcal{G}_{i, i}^{(t)} = \sum_{k=0}^{t} \nabla f_i (w^{(k)})^2.
\end{align*}
Then GMD corresponds to Adagrad.  In this case, we can apply Theorem \ref{thm: Theorem 1 - Linear Convergence of GGD} to establish linear convergence of Adagrad under the PL Inequality provided that $\phi^{(t)}$ satisfies the condition of Theorem~\ref{thm: Theorem 1 - Linear Convergence of GGD}.  The following  corollary proves that this condition holds and hence that Adagrad converges linearly. 

\begin{corollary}
\label{prop: Adagrad Linear Convergence}
Let $f: \mathbb{R}^{d} \rightarrow \mathbb{R}$ be an $L$-smooth function that is $\mu$-PL.  Let ${\alpha_l^{(t)}}^2 = \min_{i \in [d]} \mathcal{G}_{i,i}^{(t)}$ and ${\alpha_u^{(t)}}^2 = \max_{i \in [d]} \mathcal{G}_{i,i}^{(t)}$.  If $\lim\limits_{t \to \infty} \frac{\alpha_l^{(t)}}{\alpha_u^{(t)}} \neq 0$, then Adagrad converges linearly for adaptive step size $\eta^{(t)} = \frac{\alpha_l^{(t)}}{L }$.
\end{corollary}

The proof is presented in Supplementary \ref{appendix G: Proof of Corollary prop: Adagrad Linear Convergence}.  While Corollary~\ref{prop: Adagrad Linear Convergence} can be extended to the stochastic setting via Theorem~\ref{thm: Stochastic GGD linear convergence}, it requires knowledge of $\mu$ to setup the learning rate, and the resulting learning rate provided is typically smaller than what we can use in practice.  We analyze this case further in Section \ref{sec: Experiments}.  Additionally, since the condition $\lim_{t \to \infty} \frac{\alpha_l^{(t)}}{\alpha_u^{(t)}} \neq 0$ is difficult to verify in practice, we provide Corollary \ref{corollary: Adagrad Alternate Corollary} below (proof in Supplementary \ref{appendix G: Proof of Corollary prop: Adagrad Linear Convergence}), which presents a verifiable condition under which Adagrad converges linearly.

\begin{corollary}
\label{corollary: Adagrad Alternate Corollary}
Let $f: \mathbb{R}^{d} \rightarrow \mathbb{R}$ be an $L$-smooth, non-negative function that is $\mu$-PL\textsuperscript{*}.  Let ${\alpha_l^{(t)}}^2 = \min_{i \in [d]} \mathcal{G}_{i,i}^{(t)}$ and ${\alpha_u^{(t)}}^2 = \max_{i \in [d]} \mathcal{G}_{i,i}^{(t)}$.  Then Adagrad converges linearly for adaptive step size $\eta^{(t)} = \frac{\alpha_l^{(t)}}{L }$ or fixed step size $\eta = \frac{\alpha_l^{(0)}}{L}$ if $f(w^{(0)}) \leq \frac{{\alpha_l^{(0)}}^2\mu}{2L^2}$.
\end{corollary}

\section{Local Convergence of GMD}
\label{sec: Implicit Regularization}

In the previous sections, we established linear convergence for GMD for real-valued loss, $f: \mathbb{R}^{d} \to \mathbb{R}$, that is $\mu$-PL for all $x \in \mathbb{R}^{d}$.  In this section, we show that $f$ need only satisfy the PL inequality locally in order to establish linear convergence.  The following theorem (proof in Supplementary \ref{appendix E: Proof of Theorem Local convergence of GMD}) extends Theorem 4.2 by \cite{PLAndNTKBelkin} to GMD and uses the PL\textsuperscript{*} condition to establish both the existence of a global minimum and linear convergence to this global minimum under GMD\footnote{We require additional assumptions on $\phi^{(t)}$ for the case of time-dependent mirrors (see Supplementary \ref{appendix E: Proof of Theorem Local convergence of GMD}.)}.  We use $\mathcal{B}(x, R) = \{z ~; ~ z \in \mathbb{R}^{d},~ \|x - z \|_2 \leq R \}$ to denote the ball of radius $R$ centered at $x$.  

\begin{theorem}
\label{thm: Local convergence of GMD}
Suppose $\phi: \mathbb{R}^d \rightarrow \mathbb{R}^d$ is an invertible, $\alpha_u$-Lipschitz function and that $f: \mathbb{R}^{d} \rightarrow \mathbb{R}$ is non-negative, $L$-smooth, and $\mu$-PL\textsuperscript{*} on $\tilde{\mathcal{B}} = \{x  ~;~ \phi(x) \in \mathcal{B}(\phi(w^{(0)}), R)\}$ with $R = \frac{2 \sqrt{2L} \sqrt{f(w^{(0)})} \alpha_u^2 }{\alpha_l \mu}$.   If for all $x, y \in \mathbb{R}^d$ there exists $\alpha_l > 0$ such that
\begin{align*}
    \langle \phi(x) - \phi(y), x - y \rangle \geq  \alpha_l \| x -y \|^2 ,
\end{align*}
then, 
\begin{align*}
    & \text{1.} ~~ \text{there exists a global minimum $w^{(\infty)} \in \tilde{\mathcal{B}}$;} \\
    & \text{2.} ~~ \text{GMD converges linearly to $w^{(\infty)}$ for $\eta = \frac{\alpha_l}{L}$.}
\end{align*}
\end{theorem}

For the case of mirror descent, Theorem \ref{thm: Local convergence of GMD} reduces to the following corollary with the proof presented in Supplementary~\ref{appendix F: Proof of Corollary  Mirror Descent Implicit Regularization}.

\begin{corollary}
\label{cor: Mirror Descent Implicit Regularization}
Suppose $\psi$ is an $\alpha_l$-strongly convex function and that $\nabla \psi$ is $\alpha_u$-Lipschitz.   Let $D_{\psi}(x, y) = \psi(x) - \psi(y) - \nabla \psi(y)^T (x - y)$ denote the Bregman divergence for $x, y \in \mathbb{R}^d$.  If $f: \mathbb{R}^{d} \rightarrow \mathbb{R}$ is non-negative, $L$-smooth, and $\mu$-PL\textsuperscript{*} on  $\tilde{\mathcal{B}} = \{x  ~;~ \nabla \psi (x) \in \mathcal{B}(\nabla \psi (w^{(0)}), R)\}$ with $R = \frac{2 \sqrt{2L} \sqrt{f(w^{(0)})} \alpha_u^2 }{\alpha_l \mu}$, then:
\begin{enumerate}
    \item there exists a global minimum $w^{(\infty)} \in \tilde{\mathcal{B}}$ such that $D_{\psi}(w^{(\infty)}, w^{(0)}) \leq \frac{R^2}{2\alpha_l}$;
    \item mirror descent with potential $\psi$ converges linearly to $w^{(\infty)}$ for $\eta = \frac{\alpha_l}{L}$. 
\end{enumerate} 
\end{corollary}

\textbf{Remarks.} Importantly, the existence of $\tilde{\mathcal{B}}$ implies the existence of a region $\mathcal{B}' = \{x ~; D_{\psi}(w^{(\infty)}, x) \leq \frac{R^2}{2\alpha_l}\}$ with $w^{(0)} \in \mathcal{B}'$. Given $w$ with $f(w) = 0$,  \cite{MirrorDescentNonlinear} assume that there exists a region $\{w' ; D_{\psi}(w, w') \leq \epsilon\}$ containing $w^{(0)}$ in order to establish convergence of mirror descent using the fundamental identity of stochastic mirror descent.  Our result above makes explicit the value of $\epsilon$ considered in \cite{MirrorDescentNonlinear}.

\begin{figure*}[!t]
    \centering
    \includegraphics[width=\textwidth]{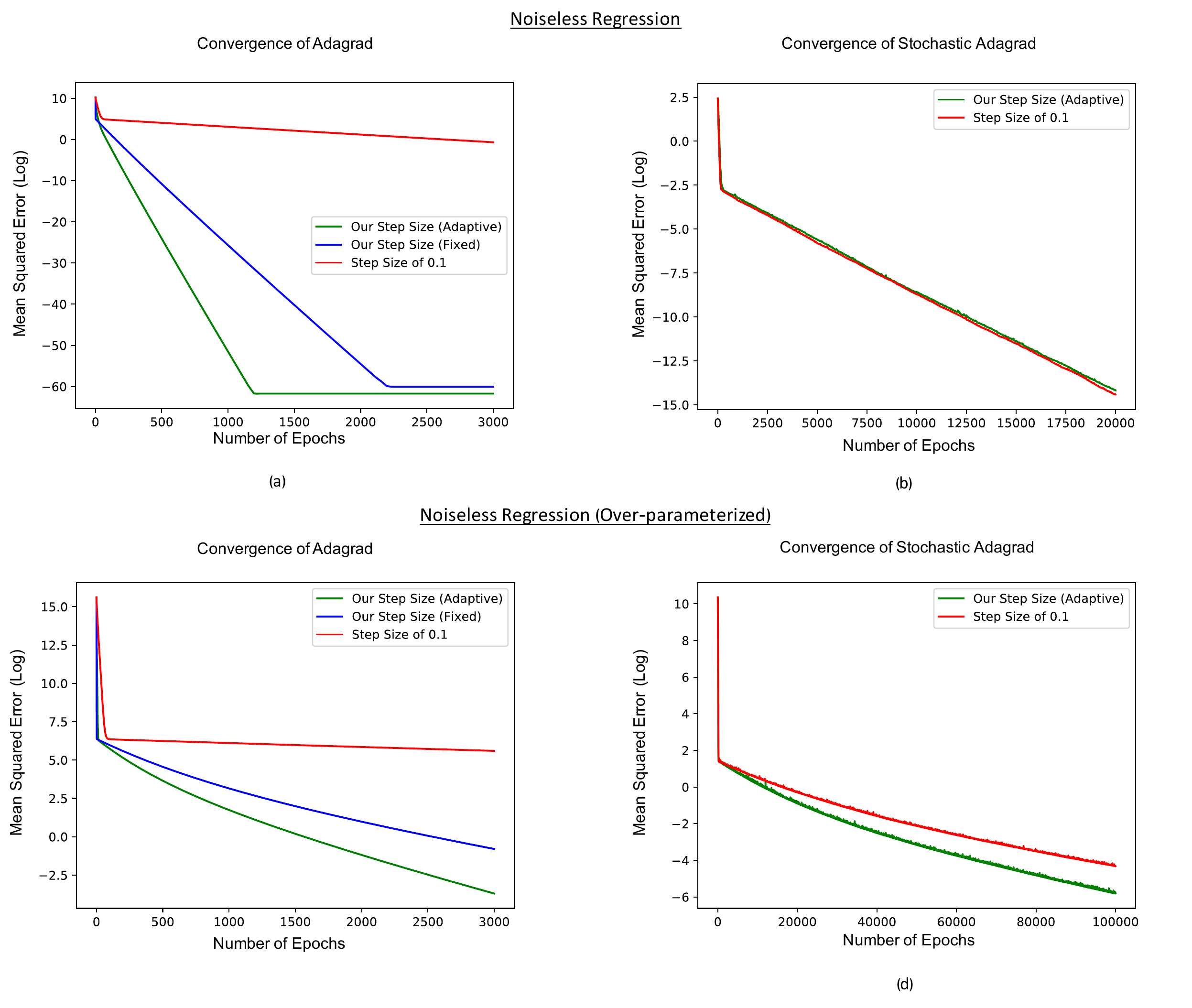}
    \caption{Using the rates provided by Corollary \ref{prop: Adagrad Linear Convergence} leads to linear convergence for (Stochastic) Adagrad in the noiseless linear regression setting also considered in \cite{AdagradPL}.  (a, b) Noiseless linear regression on 2000 examples in 20 dimensions. (c, d) Noiseless linear regression on 200 examples in 1000 dimensions.}
    \label{fig:Convergence in Noiseless Regression}
\end{figure*}

\section{Experimental Verification of our Theoretical Results}
\label{sec: Experiments}

We now present a simple set of experiments under which we can explicitly compute the learning rates in our theorems. We will show that in accordance with our theory, both fixed and adaptive versions of these learning rates yield linear convergence.  We focus on computing learning rates for Adagrad in the noiseless regression setting used in \cite{AdagradPL}.  Namely, we are given $(X, y) \in \mathbb{R}^{n \times d} \times \mathbb{R}^{n}$ such that there exists a $w^* \in \mathbb{R}^{d}$ such that $X w^* = y$.  If $n < d$, then the system is over-parameterized, and if $n \geq d$, the system is sufficiently parameterized and has a unique solution.  

In this setting, the squared loss (MSE) is $L$-smooth with $L=\lambda_{\max} (XX^T)$, and it is $\mu$-PL with $\mu = \lambda_{\min}(XX^T)$ where $\lambda_{max}$ and $\lambda_{min}$ refer to the largest and smallest non-zero eigenvalues, respectively\footnote{We take $\mu$ as the smallest non-zero eigenvalue since Adagrad updates keep parameters in the span of the data.}.  Moreover, for Adagrad, we can compute $\alpha_l^{(t)} = \min_{i \in [d]} (\sum_{k=0}^{t} \nabla f_i(w^{(k)}))^2$ and $\alpha_u^{(t)}=\max_{i \in [d]} (\sum_{k=0}^{t} \nabla f_i(w^{(k)}))^2$ at each timestep.  Hence for Adagrad in the noiseless linear regression setting, we can explicitly compute the learning rate provided in Theorem \ref{thm: Stochastic GGD linear convergence} for the stochastic setting and in Corollary \ref{prop: Adagrad Linear Convergence} for the full batch setting.  

Figure \ref{fig:Convergence in Noiseless Regression} demonstrates that in both, the over-parameterized and sufficiently parameterized settings, our provided learning rates yield linear convergence.  In the stochastic setting, the theory for fixed learning rates suggests a very small rate ($\approx10^{-9}$ for Figure \ref{fig:Convergence in Noiseless Regression}d) and hence we chose to only present the more reasonable adaptive step size as a comparison.  In the full batch setting, the learning rate obtained from our theorems out-performs using the standard fixed learning rate of $0.1$, while performance is comparable for the stochastic setting.  Interestingly, our theory suggests an adaptive learning rate that is increasing (in contrast to the usual decreasing learning rate schedules).  In particular, while the learning rate for Figure \ref{fig:Convergence in Noiseless Regression}a starts at $0.99$, it  increases to $1.56$ at the end of training.  

In Supplementary \ref{appendix H: Experiments on Over-parameterized Neural Networks}, we present experiments on over-parameterized neural networks.  While the PL condition holds in this setting \cite{PLAndNTKBelkin}, it can be difficult to compute the smoothness parameter $L$ (which was the motivation for developing Adagrad-Norm). Interestingly, our experiments demonstrate that our increasing adaptive learning rate from Theorem \ref{thm: Theorem 1 - Linear Convergence of GGD}, using an approximation for $L$, provides convergence for Adagrad in over-parameterized networks. The link to the code is provided in Supplementary \ref{appendix H: Experiments on Over-parameterized Neural Networks}.

\section{Conclusion}

In this work, we demonstrated that a PL-based analysis can be used to establish linear convergence for (stochastic) generalized mirror descent, which encompasses a range of optimization methods including gradient descent, mirror descent, and pre-conditioner methods such as Adagrad.  We first showed that the standard PL analysis for gradient descent can be extended to GMD for the non-stochastic setting under additional assumptions on the mirror function $\phi^{(t)}$.  Unfortunately, the standard PL analysis does not easily extend to the stochastic setting (SGMD). To establish linear convergence of SGMD, we developed a novel Taylor-series based analysis.  Lastly, we established a local convergence result for GMD by proving the existence of an interpolating solution in a ball around the initialization and then proving linear convergence of GMD to this solution.  Our local convergence result generalizes that of \cite{PLAndNTKBelkin} for gradient descent.  For the case of mirror descent, our local convergence result provides a formula for the radius of the ball around the initialization in Bregman divergence that contains an interpolating solution.  A bound on this radius was one of the main assumptions used to prove approximate implicit regularization results by \cite{MirrorDescentNonlinear}.


Looking ahead, we envision that the generality of our analysis (and the PL condition) could provide useful in the analysis of other commonly used adaptive methods such as Adam \cite{Adam}. Moreover, since the PL condition holds in varied settings including over-parameterized neural networks \cite{PLAndNTKBelkin}, it would be interesting to analyze whether the analysis here is compatible with that of \cite{PLAndNTKBelkin} and whether the learning rates obtained here also provide an improvement for convergence of stochastic gradient descent and mirror descent for over-parameterized neural networks.

\bibliographystyle{plain}
\bibliography{references}

\clearpage
\renewcommand{\thesubsection}{\Alph{subsection}}
\onecolumn

\newpage

\section*{Appendix} \pdfbookmark[1]{Appendix}{sec:appendix}
\renewcommand{\thesubsection}{\Alph{subsection}}

\subsection{Proof of Theorem \ref{thm: Theorem 2 - Taylor Series Analysis}}
\label{appendix B: Proof of Theorem 2}

We repeat the theorem below for convenience.   
\begin{theorem*}
Suppose $f: \mathbb{R}^{d} \rightarrow \mathbb{R}$ is $L$-smooth and $\mu$-PL and $\phi: \mathbb{R}^{d} \rightarrow \mathbb{R}^d$ is an analytic function with analytic inverse, $\phi^{-1}$.  If there exist $\alpha_l, \alpha_u > 0$ such that:
\begin{align*}
    &(a) ~~ \alpha_l \mathbf{I}  \preceq \mathbf{J}_{\phi} \preceq \alpha_u \mathbf{I},\\
    &(b) ~~ |\partial_{i_1, \ldots i_k} \phi_j^{-1} (x)| \leq \frac{k!}{2\alpha_u d} ~~ \forall x \in \mathbb{R}^d, i_1, \ldots i_k \in [d], j \in [d], k \geq 2,
\end{align*}
then GMD converges linearly for $\eta^{(t)} < \min\left( \frac{4\alpha_l^2}{5L \alpha_u}, \frac{1}{2\sqrt{d} \| \nabla f(w^{(t)}) \|} \right)$.  
\end{theorem*}

\begin{proof}
Since $f$ is $L$-smooth, it holds by Lemma that \ref{lemma: Properties of $L$-smooth functions}:
\begin{align*}
    f(w^{(t+1)}) -  f(w^{(t)}) &\leq \langle \nabla f(w^{(t)}), w^{(t+1)} - w^{(t)} \rangle + \frac{L}{2} \| w^{(t+1)} - w^{(t)} \|^2.
\end{align*}
Next, we want to bound the two quantities on the right hand side by a multiple of $\|\nabla f(w^{(t)})\|^2$.  We do so by expanding $w^{(t+1)} - w^{(t)}$ using the Taylor series for $\phi^{-1}$ as follows.  Let 
\begin{align*}
    \mathbf{v}^{(t)} := \sum_{k=2}^{\infty} \frac{1}{k!}\begin{bmatrix} \sum_{i_1, i_2 \ldots i_k =  1}^{d}  (-\eta)^k \partial_{i_1, \ldots i_k} \phi_j^{-1} (\phi(w^{(t)})) (\nabla f(w^{(t)})_{i_1} \ldots \nabla f(w^{(t)})_{i_k} )\end{bmatrix} ~,
\end{align*}
where the quantity in brackets is a column vector where we only written out the $j^{th}$ coordinate for $j \in [d]$.  Then, we have: 
\begin{align*}
    w^{(t+1)} - w^{(t)} &= \phi^{-1}(\phi(w^{(t)}) - \eta \nabla f(w^{(t)})) - w^{(t)} \\
    &=  - \eta \mathbf{J}_{\phi^{-1}}(\phi(w^{(t)})) \nabla f(w^{(t)}) + \mathbf{v}^{(t)}.
    \end{align*}
  Now we bound the term $\langle \nabla f(w^{(t)}), w^{(t+1)} - w^{(t)} \rangle$: 
\begin{align*}
    \langle \nabla f(w^{(t)}), w^{(t+1)} - w^{(t)} \rangle &= - \eta \nabla f(w^{(t)})^T  \mathbf{J}_{\phi}^{-1}(w^{(t)}) \nabla f(w^{(t)}) + \nabla f(w^{(t)})^T \mathbf{v}^{(t)}.
\end{align*}
We have separated the first order term from the other orders because we will bound them separately using conditions (a) and (b) respectively.  Namely, we first have:
\begin{align*}
    - \eta \nabla f(w^{(t)})^T  \mathbf{J}_{\phi}^{-1}(w^{(t)}) \nabla f(w^{(t)}) \leq - \frac{\eta}{\alpha_u} \|\nabla f(w^{(t)})\|^2.
\end{align*}
Next, we use the Cauchy-Schwarz inequality on inner products to bound the inner product of $\nabla f(w^{(t)})$ and the higher order terms.  In the following, we use $\alpha$ to denote $\frac{1}{2\alpha_u d}$. 
\begin{align*}
    & \nabla f(w^{(t)})^T \mathbf{v}^{(t)}  \leq  \|\nabla f(w^{(t)})\| \| \mathbf{v}^{(t)}\| \\
    \leq & \|\nabla f(w^{(t)})\|   \sum_{k=2}^{\infty} \frac{\alpha k!}{k!} (\eta)^k \left\| \begin{bmatrix} \sum_{i_1, i_2 \ldots i_k =  1}^{d}   (|\nabla f(w^{(t)})_{i_1}| \ldots |\nabla f(w^{(t)})_{i_k}| )\end{bmatrix} \right\|\\
    = & \|\nabla f(w^{(t)})\|  \alpha \sum_{k=2}^{\infty} \sqrt{d} (\eta)^k (|\nabla f(w^{(t)})_1| + \ldots |\nabla f(w^{(t)}))_d| )^k\\
    = & \|\nabla f(w^{(t)})\|  \alpha \sum_{k=2}^{\infty} (\eta)^k \sqrt{d}  |\langle \begin{bmatrix}|\nabla f(w^{(t)})_1| \\ \vdots \\ |\nabla f(w^{(t)})_d| \end{bmatrix}, \mathbf{1} \rangle |^k \\
    \leq & \|\nabla f(w^{(t)})\| \alpha \sum_{k=2}^{\infty}  (\eta)^k \sqrt{d} \| \nabla f(w^{(t)}) \|^k (\sqrt{d})^k \\
    = & \alpha \sum_{k=2}^{\infty} (\sqrt{d})^{k+1}  (\eta)^k \| \nabla f(w^{(t)}) \|^{k+1} \\
    = & \alpha (\sqrt{d})^{3}  (\eta)^2 \| \nabla f(w^{(t)}) \|^{3} \sum_{k=0}^{\infty} (\sqrt{d})^{k}  (\eta)^k \| \nabla f(w^{(t)}) \|^{k} = \frac{\alpha (\sqrt{d})^{3}  (\eta)^2 \| \nabla f(w^{(t)}) \|^{3}}{1 - \sqrt{d}\eta \|\nabla f(w^{(t)})\| }.
\end{align*}
Hence we can select $\eta < \frac{1}{2\sqrt{d} \|\nabla f(w^{(t)})\|}$ such that:
\begin{align*}  
    \frac{\alpha (\sqrt{d})^{3}  (\eta)^2 \| \nabla f(w^{(t)}) \|^{3}}{1 - \sqrt{d}\eta \|\nabla f(w^{(t)})\| }  \leq \frac{\alpha (\sqrt{d})^{3}  (\eta)^2 \| \nabla f(w^{(t)}) \|^{3}}{\sqrt{d}\eta \|\nabla f(w^{(t)})\| } = d \alpha \eta \|\nabla f(w^{(t)})\|^2.
\end{align*}
Thus, we have established the following bound:
\begin{align*}
    \langle \nabla f(w^{(t)}), w^{(t+1)} - w^{(t)} \rangle  \leq \left( -\frac{\eta}{\alpha_u} + d \alpha \eta \right)\|\nabla f(w^{(t)})\|^2 = \left( -\frac{\eta}{2\alpha_u}\right)\|\nabla f(w^{(t)})\|^2.
\end{align*}

Proceeding analogously as above, we establish a bound on $\|w^{(t+1)} - w^{(t)}\|^2$:
\begin{align*}
    \|w^{(t+1)} - w^{(t)}\|^2 \leq \left( \frac{\eta^2}{\alpha_l^2} + \alpha^2 d^2 \eta^2 \right) \|\nabla f(w^{(t)})\|^2 = \left( \frac{\eta^2}{\alpha_l^2} + \frac{\eta^2}{4 \alpha_u^2} \right) \|\nabla f(w^{(t)})\|^2.
\end{align*}
Putting the bounds together we obtain:
\begin{align*}
    f(w^{(t+1)}) - f(w^{(t)})  \leq \left( -\frac{\eta}{2\alpha_u} + \frac{L\eta^2}{2\alpha_l^2} + \frac{L\eta^2}{8 \alpha_u^2} \right) \|\nabla f(w^{(t)})\|^2.
\end{align*}
We select our learning rate to make the coefficient of $ \|\nabla f(w^{(t)}\|^2$ negative, and thus by the PL-inequality \eqref{PL-Inequality}, we have:
\begin{align*}
    &f(w^{(t+1)}) - f(w^{(t)})  \leq \left(-\frac{\eta}{2\alpha_u} + \frac{L\eta^2}{2\alpha_l^2} + \frac{L\eta^2}{8 \alpha_u^2} \right) 2\mu (f(w^{(t)}) - f(w^*)) \\
    \implies & f(w^{(t+1)}) - f(w^*) \leq \left( 1  -\frac{\mu\eta}{\alpha_u} + \frac{\mu L\eta^2}{\alpha_l^2} + \frac{\mu L \eta^2}{4 \alpha_u^2} \right)  (f(w^{(t)}) - f(w^*)).
\end{align*}
Hence, $w^{(t)}$ converges linearly when:
\begin{align*}
    0 < 1 -\frac{\mu\eta}{\alpha_u} + \frac{\mu L\eta^2}{\alpha_l^2} + \frac{\mu L \eta^2}{4 \alpha_u^2}< 1.
\end{align*}

To show that the left hand side is true, we analyze when the discriminant is negative.  Namely, we have that the left side holds if:
\begin{align*}
    &\frac{\mu^2}{\alpha_u^2} - \frac{4 \mu L}{\alpha_l^2} - \frac{\mu L}{\alpha_u^2} < 0 \\
    \implies & \frac{\mu}{\alpha_u^2} <  \frac{4  L}{\alpha_l^2} + \frac{L}{\alpha_u^2} \\
    \implies & \mu <  \frac{4L\alpha_u^2}{\alpha_l^2} + L.
\end{align*}
Since $\mu < L$ by Lemma \ref{lemma: mu less than L}, this is always true. The right hand side holds when $\eta < \frac{4\alpha_l^2}{5L\alpha_u}$, which holds by the assumption of the theorem, thereby completing the proof. 
\end{proof}

Note that if $f$ is non-negative and $\mu$-PL\textsuperscript{*}, then we have:
\begin{align*}
    \eta^{(t)} \leq  \frac{1}{2 \sqrt{2Ld} \sqrt{f(w^{(0)})} } \leq \frac{1}{2 \sqrt{2Ld} \sqrt{f(w^{(t)})} } \leq  \frac{1}{2 \sqrt{d} \|\nabla f(w^{(t)})\|}
\end{align*}

Hence, we can use a fixed learning rate of $\eta = \min \left(\frac{4 \alpha_l^2}{5L \alpha_u},  \frac{1}{2 \sqrt{2Ld} \sqrt{f(w^{(0)})} } \right)$ in this setting.

\subsection{Conditions for Monotonically Decreasing Gradients}
\label{appendix C: Conditions for Monotonically Decreasing Gradients}
As discussed in the remarks after Theorem 2, we can provide a fixed learning rate for linear convergence provided that the gradients are monotonically decreasing.  As we show below, this requires special conditions on the PL constant, $\mu$,  and  the smoothness constant, $L$, for $f$.   

\begin{prop}
Suppose $f: \mathbb{R}^{d} \rightarrow \mathbb{R}$ is $L$-smooth and $\mu$-PL and $\phi: \mathbb{R}^{d} \rightarrow \mathbb{R}^d$ is an infinitely differentiable, analytic function with analytic inverse, $\phi^{-1}$.  If there exist $\alpha_l, \alpha_u > 0$ such that:
\begin{align*}
    &(a) ~~ \alpha_l \mathbf{I}  \preceq \mathbf{J}_{\phi} \preceq \alpha_u \mathbf{I},\\
    &(b) ~~ |\partial_{i_1, \ldots i_k} \phi_j^{-1} (x)| \leq \frac{k!}{2\alpha_u d} ~~ \forall x \in \mathbb{R}^d, i_1, \ldots i_k \in [d], j \in [d], k \geq 2,\\
    &(c) ~~ \frac{\mu}{L} > \frac{4\alpha_u^2 + \alpha_l^2}{4\alpha_u^2 + 2\alpha_l^2},
\end{align*}
then GMD converges linearly for any $\eta < \min\left( \frac{4\alpha_l^2}{5L \alpha_u}, \frac{1}{2\sqrt{d} \|\nabla f(w^{(0)}) \|} \right)$.  
\end{prop}

\begin{proof}
Let $C = 1 -\frac{\mu\eta}{\alpha_u} + \frac{\mu L\eta^2}{\alpha_l^2} + \frac{\mu L \eta^2}{4 \alpha_u^2}$.  We follow exactly the proof of Theorem 2 except that at each timestep we need $C < \frac{\mu}{L}$ (which is less than $1$ by Lemma \ref{lemma: mu less than L}) in order for the gradients to converge monotonically since:
\begin{align*}
    \|\nabla f(w^{(t+1)}) \|^2 &\leq 2L (f(w^{(t+1)}) - f(w^*)) ~~ \text{See Lemma \ref{lemma: Properties of $L$-smooth functions}} \\
    & \leq 2L C (f(w^{(t)}) - f(w^*)) \\
    & \leq \frac{LC}{\mu} \|\nabla f(w^{(t)}) \|^2 ~~ \text{As $f$ is $\mu$-PL.} 
\end{align*}
Hence in order for $\|\nabla f(w^{(t+1)}) \|^2 < \|\nabla f(w^{(t)}) \|^2$, we need $C < \frac{\mu}{L}$.  Thus, we select our learning rate such that: 
\begin{align*}
    0 < 1 -\frac{\mu\eta}{\alpha_u} + \frac{\mu L\eta^2}{\alpha_l^2} + \frac{\mu L \eta^2}{4 \alpha_u^2}< \frac{\mu}{L}.
\end{align*}
Now, in order to have a solution to this system, we must ensure that the discriminant of the quadratic equation in $\eta$ when considering the right hand side inequality is larger than zero.  In particular we require:
\begin{align*}
    & \frac{\mu^2}{\alpha_u^2} - 4 \left(1 - \frac{\mu}{L} \right) \left(\frac{\mu L}{\alpha_l^2} + \frac{\mu L }{4 \alpha_u^2} \right) > 0 \\
    \implies & \frac{\mu}{L} > \frac{4\alpha_u^2 + \alpha_l^2}{4\alpha_u^2 + 2\alpha_l^2},
\end{align*}
which completes the proof.
\end{proof}

\subsection{Proof of Theorem 3}
\label{appendix D: Proof of Theorem 3}
We repeat the theorem below for convenience.  
\begin{theorem*}
Suppose $f(x) = \frac{1}{n}\sum_{i=1}^n f_i(x)$ where $f_i: \mathbb{R}^{d} \rightarrow \mathbb{R}$ are non-negative, $L_i$-smooth functions with $L = \sup_{i \in [n]} L_i$ and $f$ is $\mu$-PL\textsuperscript{*}.  Let $\phi: \mathbb{R}^{d} \rightarrow \mathbb{R}^d$ be an  analytic function with analytic inverse, $\phi^{-1}$.  SGMD is used to minimize $f$ according to the updates:
\begin{align*}
    \phi(w^{(t+1)}) = \phi(w^{(t)}) - \eta^{(t)} \nabla f_{i_t}(w^{(t)}),
\end{align*}
where $i_t \in [n]$ is chosen uniformly at random and $\eta^{(t)}$ is an adaptive step size. If there exist $\alpha_l, \alpha_u > 0$ such that:
\begin{align*}
    &(a) ~~ \alpha_l \mathbf{I}  \preceq \mathbf{J}_{\phi} \preceq \alpha_u \mathbf{I},\\
    &(b) ~~ |\partial_{i_1, \ldots i_k} \phi_j^{-1} (x)| \leq \frac{k! \mu}{2\alpha_u d L} ~~ \forall x \in \mathbb{R}^d, i_1, \ldots i_k \in [d], j \in [d], k \geq 2, 
\end{align*}
then SGMD with $\eta^{(t)} < \min\left( \frac{4\mu\alpha_l^2}{5L^2 \alpha_u}, \frac{1}{2\sqrt{d} \max_i \|\nabla f_{i}(w^{(t)}) \|} \right)$ converges linearly to a global minimum. 
\end{theorem*}
\begin{proof}
We follow the proof of Theorem \ref{thm: Theorem 2 - Taylor Series Analysis}.  Namely, Lemma \ref{lemma: sum is smooth} implies that $f$ is $L$-smooth and hence
\begin{align*}
    f(w^{(t+1)}) -  f(w^{(t)}) &\leq \langle \nabla f(w^{(t)}), w^{(t+1)} - w^{(t)} \rangle + \frac{L}{2} \| w^{(t+1)} - w^{(t)} \|^2.
\end{align*}
As before, we want to bound the two quantities on the right by $\|\nabla f(w^{(t)})\|^2$.  Following the bounds from the proof of Theorem \ref{thm: Theorem 2 - Taylor Series Analysis}, provided $\eta^{(t)} < \frac{1}{2\sqrt{d} \|\nabla f_{i_t}(w^{(t)}) \|}$ and letting
\begin{align*}
    &\mathbf{v^{(t)}} := \sum_{k=2}^{\infty} \frac{1}{k!} \begin{bmatrix} \sum_{i_1, i_2 \ldots i_k =  1}^{d}  (-\eta)^k \partial_{l_1, \ldots l_k} \phi_j^{-1} (\phi(w^{(t)})) (\nabla f_{i_t}(w^{(t)})_{l_1} \ldots \nabla f_{i_t}(w^{(t)})_{l_k} )\end{bmatrix}  ~, 
\end{align*}
we have that:
\begin{align*}
    \nabla f(w^{(t)})^T \mathbf{v^{(t)}} \leq \frac{\eta^{(t)} \mu}{2\alpha_u L  } \|\nabla f(w^{(t)})\| \|\nabla f_{i_t}(w^{(t)})\| 
\end{align*}

\noindent To remove the dependence of $\eta^{(t)}$ on $i_t$, we take $\eta^{(t)} < \frac{1}{2 \sqrt{d} \max_{i} \|\nabla f_i (w^{(t)})\|}$. Since $f$ is $\mu-$PL\textsuperscript{*} and $f_i$ is non-negative for all $i \in [n]$, $\|\nabla f_i(w^{(t)}\| \leq \sqrt{2 L f_i(w^{(t)})}$.  Thus, we can take 

\begin{align*}
    \eta^{(t)} < \frac{1}{2\sqrt{2 d L n} \sqrt{f(w^{(t)})} } \leq \frac{1}{2\sqrt{d}\max_i \|\nabla f_i(w^{(t)}) \|} 
\end{align*} 


\noindent This implies the following bounds: 
\begin{align*}
    &\langle \nabla f(w^{(t)}), w^{(t+1)} - w^{(t)} \rangle  \leq -\eta^{(t)} \nabla {f(w^{(t)})} ^T \mathbf{J}_{\phi}^{-1}(w^{(t)}) \nabla f_{i_t}(w^{(t)}) \\
    &\hspace{44mm} +  \left( \frac{\eta^{(t)} \mu}{2\alpha_u L}\right)\|\nabla f(w^{(t)})\| \|\nabla f_{i_t}(w^{(t)})\| ~; \\
    &\|w^{(t+1)} - w^{(t)}\|^2 \leq \left( \frac{{\eta^{(t)}}^2}{\alpha_l^2} + \frac{{\eta^{(t)}}^2}{4 \alpha_u^2} \right) \|\nabla f_{i_t}(w^{(t)})\|^2.
\end{align*}
Putting the bounds together we obtain:
\begin{align*}
    f(w^{(t+1)}) - f(w^{(t)})  &\leq   -\eta^{(t)} \nabla {f(w^{(t)})} ^T \mathbf{J}_{\phi}^{-1}(w^{(t)}) \nabla f_{i_t}(w^{(t)}) \\
    &~~~~~~~~~~ + \left(\frac{\eta^{(t)} \mu}{2\alpha_u L}\right)\|\nabla f(w^{(t)})\| \|\nabla f_{i_t}(w^{(t)})\| \\
    &~~~~~~~~~~ + \left( \frac{{\eta^{(t)}}^2}{\alpha_l^2} + \frac{{\eta^{(t)}}^2}{4 \alpha_u^2} \right) \|\nabla f_{i_t}(w^{(t)})\|^2 \\
    &\leq -\eta^{(t)} \nabla {f(w^{(t)})} ^T \mathbf{J}_{\phi}^{-1}(w^{(t)}) \nabla f_{i_t}(w^{(t)}) \\
    &~~~~~~~~~~ + \left(\frac{\eta^{(t)} \mu }{2\alpha_u L}\right) 2L \sqrt{f(w^{(t)}) f_{i_t}(w^{(t)})} \\
    &~~~~~~~~~~ + \left( \frac{{\eta^{(t)}}^2}{\alpha_l^2} + \frac{{\eta^{(t)}}^2}{4 \alpha_u^2} \right) \|\nabla f_{i_t}(w^{(t)})\|^2
\end{align*}
Now taking expectation over $i_t$, we obtain
\begin{align*}
    \mathbb{E}[f(w^{(t+1)})] - f(w^{(t)}) &\leq  \left( -\frac{\eta^{(t)}}{\alpha_u} \right) \| \nabla {f(w^{(t)})}\|^2 + \left( \frac{\eta^{(t)} \mu}{\alpha_u}\right) \sqrt{f(w^{(t)})} \mathbb{E}\left[\sqrt{f_{i_t}(w^{(t)})}\right] \\
    &~~~~~~~~~~ + \left( \frac{L{\eta^{(t)}}^2}{2\alpha_l^2} + \frac{L{\eta^{(t)}}^2}{8 \alpha_u^2} \right) \mathbb{E}[\|\nabla f_{i_t}(w^{(t)})\|^2] \\
    & \leq  \left( -\frac{\eta^{(t)}}{\alpha_u} \right) \| \nabla {f(w^{(t)})}\|^2 + \left( \frac{\eta^{(t)} \mu}{\alpha_u}\right) f(w^{(t)})  \\
    &~~~~~~~~~~ + \left( \frac{L{\eta^{(t)}}^2}{2\alpha_l^2} + \frac{L{\eta^{(t)}}^2}{8 \alpha_u^2} \right) \mathbb{E}[\|\nabla f_{i_t}(w^{(t)})\|^2] \\
    & \leq  \left( -\frac{2 \mu \eta^{(t)}}{\alpha_u} \right)  {f(w^{(t)})} + \left( \frac{\eta^{(t)} \mu}{\alpha_u}\right) f(w^{(t)})  \\
    &~~~~~~~~~~ + \left( \frac{{L\eta^{(t)}}^2}{2\alpha_l^2} + \frac{{L\eta^{(t)}}^2}{8\alpha_u^2} \right) \mathbb{E}[2L(f_{i_t}(w^{(t)}))]\\
    &\leq \left( -\frac{\mu\eta^{(t)}}{\alpha_u} +  \frac{{L^2\eta^{(t)}}^2}{\alpha_l^2} + \frac{{L^2\eta^{(t)}}^2}{4\alpha_u^2} \right) (f(w^{(t)})).
\end{align*}
where the second inequality follows from Jensen's inequality and the third inequality follows from Lemma \ref{lemma: Properties of $L$-smooth functions}.
Hence, we have:
\begin{align*}
    \mathbb{E}[f(w^{(t+1)})]  \leq \left(1 -\frac{\mu\eta^{(t)}}{\alpha_u} +  \frac{{L^2\eta^{(t)}}^2}{\alpha_l^2} + \frac{{L^2\eta^{(t)}}^2}{4\alpha_u^2} \right) (f(w^{(t)}) ).
\end{align*}
Now let $C =  \left(-\frac{\mu\eta^{(t)}}{\alpha_u} +  \frac{{L^2\eta^{(t)}}^2}{\alpha_l^2} + \frac{{L^2\eta^{(t)}}^2}{4\alpha_u^2}\right)$.  Then taking expectation with respect to $i_t, i_{t-1}, \ldots i_1$, yields
\begin{align*}
     \mathbb{E}_{i_t, \ldots, i_1}[f(w^{(t+1)})]  &\leq (1+C) (\mathbb{E}_{i_t, \ldots, i_1}[f(w^{(t)})])  \\
     &= (1+C) (\mathbb{E}_{i_{t-1}, \ldots, i_1} [\mathbb{E}_{i_t | i_{t-1}, \ldots i_1}[f(w^{(t)})]])  \\
     &= (1+C) (\mathbb{E}_{i_{t-1}, \ldots, i_1} f(w^{(t)})] ).
\end{align*}
Hence, we can proceed inductively to conclude that
\begin{align*}
    \mathbb{E}_{i_t, \ldots, i_1}[f(w^{(t+1)})]  &\leq (1+C)^{t+1} (f(w^{(0)}))).
\end{align*}
Thus if $0 < 1+C < 1$, we establish linear convergence.  The left hand side is satisfied since $\mu < L$, and the right hand side is satisfied for $\eta^{(t)} < \frac{4\mu\alpha_l^2}{5L^2\alpha_u}$, which holds by the theorem's assumption, thereby completing the proof.
\end{proof}

\subsection{Proof of Theorem \ref{thm: Local convergence of GMD}}
\label{appendix E: Proof of Theorem Local convergence of GMD}

We restate the theorem below.
\begin{theorem*}
Suppose $\phi: \mathbb{R}^d \rightarrow \mathbb{R}^d$ is an invertible, $\alpha_u$-Lipschitz function and that $f: \mathbb{R}^{d} \rightarrow \mathbb{R}$ is non-negative, $L$-smooth, and $\mu$-PL\textsuperscript{*} on $\tilde{\mathcal{B}} = \{x  ~;~ \phi(x) \in \mathcal{B}(\phi(w^{(0)}), R)\}$ with $R = \frac{2 \sqrt{2L} \sqrt{f(w^{(0)})} \alpha_u^2 }{\alpha_l \mu}$.   If for all $x, y \in \mathbb{R}^d$ there exists $\alpha_l > 0$ such that
\begin{align*}
    \langle \phi(x) - \phi(y), x - y \rangle \geq  \alpha_l \| x -y \|^2 ,
\end{align*}
then, 
\begin{align*}
    & (1) ~~ \text{There exists a global minimum $w^{(\infty)} \in \tilde{\mathcal{B}}$.} \\
    & (2) ~~ \text{GMD converges linearly to $w^{(\infty)}$ for $\eta = \frac{\alpha_l}{L}$.} \\
\end{align*}
\end{theorem*}

\begin{proof}
The proof follows from the proofs of Lemma \ref{lemma: Linear convergence under PL*}, Theorem \ref{thm: Theorem 1 - Linear Convergence of GGD}, and Theorem 4.2 from \cite{PLAndNTKBelkin}.  Namely, we will proceed by strong induction.  Let $\kappa = \frac{L{\alpha_u}^2}{\mu {\alpha_l}^2}$.  At timestep $0$, we trivially have that $w^{(0)} \in \tilde{\mathcal{B}}$ and $f(w^{(0)}) \leq f(w^{(0)})$.  At timestep $t$, we assume that $w^{(0)}, w^{(1)}, \ldots w^{(t)} \in \tilde{\mathcal{B}}$ and that $f(w^{(i)}) \leq (1 - {\kappa}^{-1}) f(w^{(i-1)}) $ for $i \in [t]$.  Then at timestep $t+1$, from the proofs of Lemma \ref{lemma: Linear convergence under PL*} and Theorem \ref{thm: Theorem 1 - Linear Convergence of GGD}, we have:
$$f(w^{(t+1)}) \leq (1 - {\kappa}^{-1}) f(w^{(t)})$$  
Next, we need to show that $w^{(t+1)} \in \tilde{\mathcal{B}}$.  We have that:
\begin{align}
    \|\phi(w^{(t+1)}) - \phi(w^{(0)}) \| &= \left\| \sum_{i=0}^{t} -\eta \nabla f(w^{(i)}) \right \| \nonumber \\
    &\leq \eta \sum_{i=0}^{t} \| \nabla f(w^{(i)}) \| ~~ \text{By the Triangle Inequality} \nonumber\\
    &\leq \eta \sqrt{2 \frac{L \alpha_u^2}{\alpha_l^2}}\sum_{i=0}^{t}  \sqrt{f(w^{(t)}) - f(w^{(t+1)})} ~~ \label{identity 1}\\
    &\leq \eta \sqrt{2 \frac{L \alpha_u^2}{\alpha_l^2}}\sum_{i=0}^{t}  \sqrt{f(w^{(t)})} \nonumber \\
    &\leq \eta \sqrt{2L} \frac{\alpha_u}{\alpha_l} \sum_{i=0}^{t}  \sqrt{(1 - \kappa^{-1})^{i}} \sqrt{f(w^{(0)})} \nonumber \\
    &= \eta \sqrt{2 L f(w^{(0)})} \frac{\alpha_u}{\alpha_l} \sum_{i=0}^{t} (1 - \kappa^{-1})^{\frac{i}{2}} \nonumber \\
    &\leq \eta \sqrt{2 L f(w^{(0)})} \frac{\alpha_u}{\alpha_l} \frac{1}{1 - \sqrt{1 - \kappa^{-1}}} \nonumber\\
    &\leq \eta \sqrt{2 L f(w^{(0)})} \frac{\alpha_u}{\alpha_l} \frac{2}{\kappa^{-1}} \nonumber \\
    &= \frac{\alpha_l}{L} \sqrt{2 L f(w^{(0)})} \frac{\alpha_u}{\alpha_l} 2 \frac{\alpha_u L}{\alpha_l \mu} \nonumber\\
    &= \frac{2 \sqrt{2L} \sqrt{f(w^{(0)})} \alpha_u^2 }{\alpha_l \mu} = R \nonumber
\end{align} 
The identity in \eqref{identity 1} follows from the proof of $f(w^{(t+1)}) \leq (1 - \kappa^{-1}) f(w^{(t)})$. Namely,  
\begin{align*}
    & f(w^{(t+1)}) -  f(w^{(t)}) \leq -\frac{L}{2 \alpha_u^2}  \|-\eta \nabla f(w^{(t)}) \|^2 \\
    \implies & \| \nabla f(w^{(t)}) \| \leq \sqrt{\frac{2 \alpha_u^2}{L}} \sqrt{f(w^{(t)}) - f(w^{(t+1)})} \\
    \implies & \| \nabla f(w^{(t)}) \| \leq \eta \sqrt{\frac{2 L \alpha_u^2}{\alpha_l^2}} \sqrt{f(w^{(t)}) - f(w^{(t+1)})}
\end{align*}
Hence we conclude that $w^{(t+1)} \in \tilde{\mathcal{B}}$ and so induction is complete.  
\end{proof}

In the case that $\phi^{(t)}$ is time-dependent, we establish a similar convergence result by assuming that: $$\left\|\sum\limits_{i=1}^{\infty} \phi^{(i)}(w^{(i)}) - \phi^{(i-1)}(w^{(i)})\right\| = \delta < \infty$$.  Additionally if $\alpha_u^{(t)}$ has a uniform upper bound and  $\alpha_l^{(t)}$ has a uniform nonzero lower bound, then: 
\begin{align*}
    \|\phi^{(t)}(w^{(t+1)}) - \phi^{(0)}(w^{(0)})\| &= \|\phi^{(t)}(w^{(t+1)}) - \phi^{(t)}(w^{(t)}) \\
    &~~~~+ \phi^{(t)}(w^{(t)}) - \phi^{(t-1)}(w^{(t)}) \\
    & ~~~~ + \phi^{(t-1)}(w^{(t)}) - \phi^{(t-1)}(w^{(t-1)}) \\
    &~~~~+ ~\ldots ~+ \phi^{(0)}(w^{(1)}) - \phi^{(0)}(w^{(0)}) \| \\
    &\leq \left\| \sum_{i=0}^{t} \phi^{(i)}(w^{(i+1)}) - \phi^{(i)}(w^{(i)})\right\| \\
    &~~~~+  \left\| \sum_{i=1}^{t} \phi^{(i)}(w^{(i)}) - \phi^{(i-1)}(w^{(i)})\right\| \\
    &\leq R + \delta
\end{align*}
Hence we would conclude that $\phi^{(t)}(w^{(t+1)}) \in \mathcal{B}(\phi^{(0)}(w^{(0)}), R + \delta)$.

\subsection{Proof of Corollary \ref{prop: Adagrad Linear Convergence} and Corollary 2}
\label{appendix G: Proof of Corollary prop: Adagrad Linear Convergence}
We repeat Corollary \ref{prop: Adagrad Linear Convergence} below.  

\begin{corollary*}
Let $f: \mathbb{R}^{d} \rightarrow \mathbb{R}$ be an $L$-smooth function that is $\mu$-PL.  Let ${\alpha_l^{(t)}}^2 = \min_{i \in [d]} \mathcal{G}_{i,i}^{(t)}$ and ${\alpha_u^{(t)}}^2 = \max_{i \in [d]} \mathcal{G}_{i,i}^{(t)}$.  If $\lim\limits_{t \to \infty} \frac{\alpha_l^{(t)}}{\alpha_u^{(t)}} \neq 0$, then Adagrad converges linearly for adaptive step size $\eta^{(t)} = \frac{\alpha_l^{(t)}}{L }$.
\end{corollary*}

\begin{proof}
By definition of $\mathcal{G}^{(t)}$, we have that:
\begin{align*}
    &(1) ~~ {\alpha_l^{(t)}}^2 = \min_{i \in [d]} \mathcal{G}_{i,i}^{(t)}\\
    &(2) ~~ {\alpha_u^{(t)}}^2 = \max_{i \in [d]} \mathcal{G}_{i,i}^{(t)}
\end{align*}

From the proof of Theorem \ref{thm: Theorem 1 - Linear Convergence of GGD}, using learning rate $\eta^{(t)} = \frac{\alpha_l^{(t)}}{L }$ at timestep $t$ gives:
\begin{align*}
    f(w^{(t+1)}) - f(w^*) \leq \left(1 - \frac{\mu{\alpha_l^{(t)}}^2}{L{\alpha_u^{(t)}}^2} \right) (f(w^{(t)}) - f(w^*))
\end{align*}

Let $\kappa^{(t)} = \frac{\mu{\alpha_l^{(t)}}^2}{L{\alpha_u^{(t)}}^2}$.  Although we have that $(1 - \kappa^{(t)}) < 1$ for all $t$, we need to ensure that $\prod\limits_{i=0}^{\infty}(1 - \kappa^{(i)}) = 0$ (otherwise we would not get convergence to a global minimum).  Using the assumption that $\lim\limits_{t \to \infty} \frac{\alpha_l^{(t)}}{\alpha_u^{(t)}} \neq 0$, let $\lim\limits_{t \to \infty}(1 - \kappa^{(t)}) = 1 - c < 1$.  Then using the definition of the limit, for $0 < \epsilon < c$, there exists $N$ such that for $t > N$, $\left| \kappa^{(t)} -  c\right| < \epsilon$.  Hence, letting $c^* = \min \left(c - \epsilon, \min\limits_{t \in \{0, 1, \ldots N\}}  \kappa^{(t)} \right)$, implies that $(1 - \kappa^{(t)}) < 1 - c^*$ for all timesteps $t$. Thus, we have that:
\begin{align*}
    \prod\limits_{i=0}^{\infty}(1 - \kappa^{(i)}) < \prod\limits_{i=0}^{\infty} (1 - c^*) = 0
\end{align*}
Thus, Adagrad converges linearly to a global minimum.   
\end{proof}

\noindent   We present Corollary \ref{corollary: Adagrad Alternate Corollary} below.

\begin{corollary*}
Let $f: \mathbb{R}^{d} \rightarrow \mathbb{R}$ be an $L$-smooth, non-negative function that is $\mu$-PL\textsuperscript{*}.  Let ${\alpha_l^{(t)}}^2 = \min_{i \in [d]} \mathcal{G}_{i,i}^{(t)}$.  Then Adagrad converges linearly for adaptive step size $\eta^{(t)} = \frac{\alpha_l^{(t)}}{L }$ or fixed step size $\eta = \frac{\alpha_l^{(0)}}{L}$ if:
\begin{align*}
    f(w^{(0)}) < \frac{\alpha_l^2 \mu}{2L^2}
\end{align*}
\end{corollary*}

\begin{proof}
By definition of $\mathcal{G}^{(t)}$, we have that:
\begin{align*}
    &(1) ~~ {\alpha_l^{(t)}}^2 = \min_{i \in [d]} \mathcal{G}_{i,i}^{(t)}\\
    &(2) ~~ {\alpha_u^{(t)}}^2 = \max_{i \in [d]} \mathcal{G}_{i,i}^{(t)}
\end{align*}

In particular, we can choose $\alpha_l = \alpha_l^{(0)}$ uniformly.  We need to now ensure that $\alpha_u^{(t)}$ does not diverge.  We prove this by using strong induction to show that $\frac{{\alpha_l^{(t)}}^2}{{\alpha_u^{(t)}}^2} \geq  S$ uniformly for $S = \frac{\mu \alpha_l^2 - 2L^2 f(w^{(0)})}{\mu \alpha_u^2}$ (where the assumption on $f(w^{(0)})$ ensures that the numerator is positive).  The base case holds since we have $\frac{{\alpha_l^{(0)}}^2}{{\alpha_u^{(0)}}^2} \geq  S $ by definition. 
Now assume that $\frac{{\alpha_l^{(i)}}^2}{{\alpha_u^{(i)}}^2} \geq S$ for $i \in \{0, 1, \ldots t-1\}$.  Then we have:
\begin{align*}
    \frac{{\alpha_l^{(t)}}^2}{{\alpha_u^{(t)}}^2} &\geq \frac{{\alpha_l^{(0)}}^2}{{\alpha_u^{(t)}}^2} \\
    &\geq \frac{{\alpha_l^{(0)}}^2}{{\alpha_u^{(t-1)}}^2 + \|\nabla f(w^{(t)})\|^2}\\
    &\geq \frac{{\alpha_l^{(0)}}^2}{{\alpha_u^{(t-1)}}^2 + 2Lf(w^{(t)})} ~~~~ \text{(By Lemma \ref{lemma: Properties of $L$-smooth functions}b)}\\
    &\geq \frac{{\alpha_l^{(0)}}^2}{{\alpha_u^{(0)}}^2 + 2Lf(w^{(0)}) \sum_{j=1}^{t-1} \prod_{k=0}^{j}\left(1 - \frac{\mu{\alpha_l^{(k)}}^2}{L{\alpha_u^{(k)}}^2}\right)} \\
    &\geq \frac{{\alpha_l^{(0)}}^2}{{\alpha_u^{(0)}}^2 + 2Lf(w^{(0)}) \sum_{j=1}^{t-1} \left(1 - \frac{\mu S}{L}\right)
    ^j} \\
    &\geq \frac{{\alpha_l^{(0)}}^2}{{\alpha_u^{(0)}}^2 + 2Lf(w^{(0)}) \frac{1}{1 - 1 + \frac{\mu  S}{L}}} ~~~~ \text{$\left( \text{Since}~~ 0 < S < \frac{L}{\mu}\right)$} \\
    &= S
\end{align*}

Hence, by induction, $\frac{{\alpha_l^{(t)}}^2}{{\alpha_u^{(t)}}^2} \geq S$ for all timesteps $t$.  

\end{proof}

\subsection{Proof of Corollary \ref{cor: Mirror Descent Implicit Regularization}}
\label{appendix F: Proof of Corollary  Mirror Descent Implicit Regularization}

We present the corollary below.

\begin{corollary}
Suppose $\psi$ is an $\alpha_l$-strongly convex function and that $\nabla \psi$ is $\alpha_u$-Lipschitz.   Let $D_{\psi}(x, y) = \psi(x) - \psi(y) - \nabla \psi(y)^T (x - y)$ denote the Bregman divergence for $x, y \in \mathbb{R}^d$.  If $f: \mathbb{R}^{d} \rightarrow \mathbb{R}$ is non-negative, $L$-smooth, and $\mu$-PL\textsuperscript{*} on $\tilde{\mathcal{B}} = \{x  ~;~ \nabla \psi (x) \in \mathcal{B}(\nabla \psi (w^{(0)}), R)\}$ with $R = \frac{2 \sqrt{2L} \sqrt{f(w^{(0)})} \alpha_u^2 }{\alpha_l \mu}$, then:
\begin{align*}
    & (1) ~~ \text{There exists a global minimum $w^{(\infty)} \in \tilde{\mathcal{B}}$ such that $D_{\psi}(w^{(\infty)}, w^{(0)}) \leq \frac{R^2}{2\alpha_l}$}. \\
    & (2) ~~ \text{Mirror descent with potential $\psi$ converges linearly to $w^{(\infty)}$ for $\eta = \frac{\alpha_l}{L}$}. \\
\end{align*} 
\end{corollary}

\begin{proof}
The proof of existence and linear convergence follow immediately from Theorem \ref{thm: Local convergence of GMD}.  All that remains is to show that $D_{\psi}(w^{(\infty)}, w^{(0)}) \leq \frac{R^2}{2\mu}$.  As $\psi$ is $\alpha_l$-strongly convex, we have:
\begin{align*}
    & \psi(w^{(\infty)}) \leq \psi(w^{(0)}) + \langle \nabla \psi (w^{(0)}), w^{(\infty)} -  w^{(0)} \rangle \\
    &~~~~~~~~~~~~~~ + \frac{1}{2\alpha_l} \|\nabla \psi (w^{(\infty)}) - \nabla \psi(w^{(0)}) \|^2 ~~ \text{By Lemma \ref{lemma: Strong Convexity Upper Bound}}\\
    \implies & D_{\psi}(w^{(\infty)}, w^{(0)}) \leq \frac{1}{2\alpha_l} \|\nabla \psi (w^{(\infty)}) - \nabla \psi(w^{(0)}) \|^2 \leq \frac{R^2}{2\alpha_l}
\end{align*}
\end{proof}

\subsection{Experiments on Over-parameterized Neural Networks}
\label{appendix H: Experiments on Over-parameterized Neural Networks}
Below, we present experiments in which we apply the learning rate given by Corollary \ref{prop: Adagrad Linear Convergence} to over-parameterized neural networks.  Since the main difficulty is estimating the parameter $L$ in neural networks, we instead provide a crude approximation for $L$ by setting $L^{(t)} = .99 \frac{\|\nabla f(w^{(t)})\|^2}{2 f(w^{(t)})}$.  The intuition for this approximation comes from Lemma \ref{lemma: Properties of $L$-smooth functions}.  While there are no guarantees that this approximation yields linear convergence according to our theory, Figure \ref{fig:Over-parameterized Networks} suggests empirically that this approximation provides convergence.  Moreover, this approximation allows us to compute our adaptive learning rate in practice.  

\begin{figure}[!h]
    \centering
    \includegraphics[width=1\textwidth]{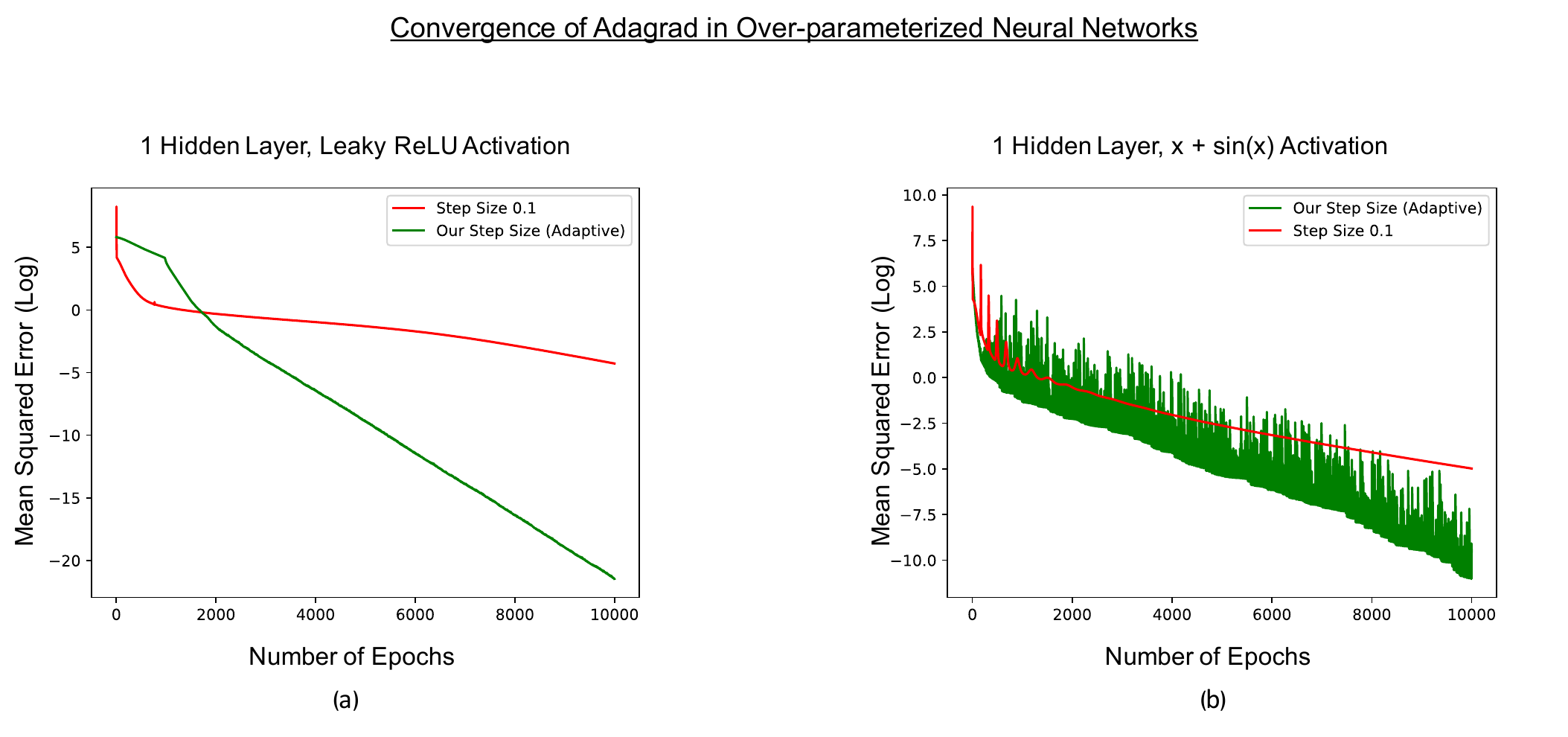}
    \caption{Using the adaptive rate provided by Corollary \ref{prop: Adagrad Linear Convergence} with $L$ approximated by  $L^{(t)} = .99 \frac{\|\nabla f(w^{(t)}\|^2}{2 f(w^{(t)})}$ leads to convergence for  Adagrad in the noisy linear regression setting (60 examples in 50 dimensions with uniform noise applied to the labels).  (a) 1 hidden layer network with Leaky ReLU activation \cite{LeakyReLU} and 100 hidden units. (b) 1 hidden layer network with $x + \sin(x)$ activation with 100 hidden units. All networks were trained using a single Titan Xp, but can be trained on a laptop as well.}
    \label{fig:Over-parameterized Networks}
\end{figure}

Code for all experiments is available at:

\url{https://anonymous.4open.science/r/cef30260-473d-4116-bda1-1debdcc4e00a/}

\subsection{Alternate Conditions for Linear Convergence in SGMD}
\label{appendix I: alternate conditions for linear convergence in SGMD}

We now briefly discuss the difficulty in extending Theorem \ref{thm: Theorem 1 - Linear Convergence of GGD} to the stochastic setting.  Instead of resorting to a Taylor series analysis, an alternate strategy is to consider convergence of the terms $z^{(t)} = \phi(w^{(t)})$ under the assumption that the Jacobian of $\phi$, $J_{\phi}$, has bounded, non-zero spectrum.  In this case, the GMD updates proceed as follows:
\begin{align*}
    z^{(t+1)} = z^{(t)} - \eta \nabla f(\phi^{-1}(z^{(t)}))
\end{align*}
Letting $g(z) = f(\phi^{-1}(z))$ for all $z \in \mathbb{R}^{d}$, we have:
\begin{align*}
    \nabla_z f(\phi^{-1}(z)) = J_{\phi}(z) \nabla_z g(z)
\end{align*}
Thus, the GMD update rule can be written as: 
\begin{align*}
    z^{(t+1)} = z^{(t)} - \eta J_{\phi}(z^{(t)}) \nabla_z g(z^{(t)})
\end{align*}

This update rule is similar to that of pre-conditioned gradient descent, and if $g$ is $\mu$-PL for some $\mu$ and $g$ is $L$-smooth, then we can easily establish linear convergence even in the stochastic setting.  However, we provide an example below that demonstrates additional conditions are necessary for concluding that $g$ is even $L$-smooth. 

\textit{Example.}
Let $f(x) = \frac{x^2}{2}$ and let $\phi^{-1}(x) = 2x + \sin x $, which  is strictly monotonically increasing and so invertible. Then, $g(x) = f(\phi^{-1}(x))$.  Note that $\phi^{-1}$ has bounded derivatives of all orders and $\phi$ has bounded derivative since:
\begin{align*}
    \frac{d\phi(x)}{dx} = \frac{1}{\frac{d\phi^{-1}(\phi(x))}{dx}} = \frac{1}{2 + \cos(\phi(x))}
\end{align*}
Additionally, $f$ is clearly $L$-smooth ($L = 1$). Now $g$ is not $L$-smooth since
\begin{align*}
    g'(x) = f'(\phi^{-1}(x)) \frac{d\phi^{-1}(x)}{dx} = (2x + \sin x)(2 + \cos x) 
\end{align*}
and $2x \cos x$ is not a Lipschitz function.  

Nevertheless, the example above satisfies the conditions of Theorem \ref{thm: Theorem 2 - Taylor Series Analysis}, and so our analysis for linear convergence still holds for this case.  
\end{document}